\documentclass{article}

    \usepackage[final]{neurips_2022}

\usepackage[utf8]{inputenc} 
\usepackage[T1]{fontenc}    
\usepackage{hyperref}       
\usepackage{url}            
\usepackage{booktabs}       
\usepackage{amsfonts}       
\usepackage{nicefrac}       
\usepackage{microtype}      
\usepackage{xcolor}         
\usepackage{tabularx}
\usepackage{graphicx}
\usepackage{subfigure}
\usepackage{mathrsfs}
\usepackage{amsmath,amssymb,amsthm,amsbsy,latexsym,dsfont,tikz}

\usepackage{enumerate}

\newcommand{\diag}[0]{\text{diag}}

\newcommand{\R}[0]{\mathbb{R}}

\theoremstyle{remark}
\newtheorem{thm}{Theorem}
\newtheorem{lem}[thm]{Lemma}
\newtheorem{cor}[thm]{Corollary}

\newtheorem{rem}[thm]{Remark}

\newtheorem{ex}[thm]{Example}
  
\newtheorem{cla}[thm]{Claim}
\renewcommand{\leq}{\leqslant} 
\renewcommand{\geq}{\geqslant}

\def\qed{ \hfill $\blacksquare$}  
\newcommand{\rd}{\mathrm{d}}

\newcommand{\cP}{\mathcal{P}}

\newcommand{\vone}{\mathbf{1}}

\newcommand{\bR}{\mathbb{R}}

\title{Wasserstein $K$-means for clustering \\ probability distributions}

\author{%
  Yubo~Zhuang,~~Xiaohui~Chen,~~Yun~Yang \\
  Department of Statistics\\
  University of Illinois at Urbana-Champaign\\
  \texttt{\{yubo2,xhchen,yy84\}@illinois.edu} \\
}

\begin{document}

\maketitle

\begin{abstract}
  Clustering is an important exploratory data analysis technique to group objects based on their similarity. The widely used $K$-means clustering method relies on some notion of distance to partition data into a fewer number of groups. In the Euclidean space, centroid-based and distance-based formulations of the $K$-means are equivalent. In modern machine learning applications, data often arise as probability distributions and a natural generalization to handle measure-valued data is to use the optimal transport metric. Due to non-negative Alexandrov curvature of the Wasserstein space, barycenters suffer from regularity and non-robustness issues. The peculiar behaviors of Wasserstein barycenters may make the centroid-based formulation fail to represent the within-cluster data points, while the more direct distance-based $K$-means approach and its semidefinite program (SDP) relaxation are capable of recovering the true cluster labels. In the special case of clustering Gaussian distributions, we show that the SDP relaxed Wasserstein $K$-means can achieve exact recovery given the clusters are well-separated under the $2$-Wasserstein metric. Our simulation and real data examples also demonstrate that distance-based $K$-means can achieve better classification performance over the standard centroid-based $K$-means for clustering probability distributions and images.
\end{abstract}

\section{Introduction}
\label{sec:intro}

Clustering is a major tool for unsupervised machine learning problems and exploratory data analysis in statistics. Suppose we observe a sample of data points $X_1,\dots,X_n$ taking values in a metric space $({\cal X}, \|\cdot\|)$. Suppose there exists a clustering structure $G_1^\ast,\dots,G_K^\ast$ such that each data point $X_i$ belongs to exactly one of the unknown cluster $G_k^\ast$. The goal of clustering analysis is to recover the true clusters $G_1^\ast,\dots,G_K^\ast$ given the input data $X_1,\dots,X_n$. In the Euclidean space ${\cal X} = \bR^p$, the $K$-means clustering is a widely used  method that achieves the empirical success in many applications~\citep{MacQueen1967_kmeans}. In modern machine learning and data science problems such as computer graphics~\citep{Solomon_2015}, data exhibits complex geometric features and traditional clustering methods developed for Euclidean data may not be well suited to analyze such data.

In this paper, we consider the clustering problem of probability measures $\mu_1,\dots,\mu_n$ into $K$ groups. As a motivating example, the MNIST dataset contains images of handwritten digits 0-9. Normalizing the greyscale images into histograms as probability measures, a common task is to cluster the images. One can certainly apply the Euclidean $K$-means to the vectorized images. However, this would lose important geometric information of the two-dimensional data. On the other hand, theory of optimal transport~\citep{Villani2003_topics-in-ot} provides an appealing framework to model measure-valued data as probabilities in many statistical tasks~\citep{Domazakis2019ClusteringMD,ChenLinMuller2021,Bigot2017_GPCA,SeguyCuturi2015_NIPS,RigolletWeed2019,HutterRigollet2019,CazellesSeguyBigotCuturiPapadakis2018}.


{\bf Background on $K$-means clustering.} Algorithmically, the $K$-means clustering have two equivalent formulations in the Euclidean space -- centroid-based and distance-based -- in the sense that they both yield the same partition estimate for the true clustering structure. Given the Euclidean data $X_1,\dots,X_n \in \R^p$, the \emph{centroid-based} formulation of standard $K$-means can be expressed as
\begin{equation}
    \label{eqn:euclidean_kmeans_center-based}
    \min_{\beta_1,\dots,\beta_K \in \R^d} \sum_{i=1}^n \min_{k \in [K]} \|X_i-\beta_k\|_2^2=\min_{G_1,\dots,G_K} \Big\{ \sum_{k=1}^K \sum_{i\in G_k} \|X_i - \bar X_k\|_2^2 : \bigsqcup_{k=1}^K G_k = [n] \Big\},
\end{equation}
where clusters $\{G_k\}_{k=1}^K$ are determined by the Voronoi diagram from $\{\beta_k\}_{k=1}^K$, $\bar X_k = |G_k|^{-1}\sum_{i\in G_k} X_i$ denotes the centroid of cluster $G_k$, $\bigsqcup$ denotes the disjoint union and $[n] = \{1,\dots,n\}$. Heuristic algorithm for solving~\eqref{eqn:euclidean_kmeans_center-based} includes Lloyd's algorithm~\citep{Lloyd1982_TIT}, which is an iterative procedure alternating the partition and centroid estimation steps. Specifically, given an initial centroid estimate $\beta_1^{(1)},\dots,\beta_K^{(1)}$, one first assigns each data point to its nearest centroid at the $t$-th iteration according to the Voronoi diagram, i.e.,
\begin{equation}
    \label{eqn:voronoi_diagram}
    G_k^{(t)} = \left\{ i \in [n] : \|X_i-\beta_k^{(t)}\|_2 \leq  \|X_i-\beta_j^{(t)}\|_2, \ \forall j \in [K] \right\},
\end{equation}
and then update the centroid for each cluster
\begin{equation}
    \label{eqn:lloyd_update}
    \beta_k^{(t+1)} = {1 \over |G_k^{(t)}|} \sum_{i \in G_k^{(t)}} X_i,
\end{equation}
where $|G_k^{(t)}|$ denotes the cardinality of $G_k^{(t)}$. Step~\eqref{eqn:voronoi_diagram} and step~\eqref{eqn:lloyd_update} alternate until convergence.

The \emph{distance-based} (sometimes also referred as \emph{partition-based}) formulation directly solves the following constrained optimization problem without referring to the estimated centroids:
\begin{equation}
    \label{eqn:euclidean_kmeans_partition-based}
        \min_{G_1,\dots,G_K} \Big\{ \sum_{k=1}^K {1\over|G_k|} \sum_{i,j \in G_k} \|X_i-X_j\|_2^2 : \bigsqcup_{k=1}^K G_k = [n] \Big\}.
\end{equation}
Observe that~\eqref{eqn:euclidean_kmeans_center-based} with nearest centroid assignment and~\eqref{eqn:euclidean_kmeans_partition-based} are equivalent for the clustering purpose due to the following identity, which extends the parallelogram law from two points to $n$ points,
\begin{align}\label{eqn:sum_of_squares}
   \sum_{i,j=1}^n\|X_i-X_j\|_2^2 =2n \sum_{i=1}^n \|X_i - \bar X\|_2^2, \quad\mbox{with}\quad \bar X=\frac{1}{n}\sum_{i=1}^n X_i \quad\mbox{and}\quad X_i\in\R^p.
\end{align}
Consequently, the two criteria yield the same partition estimate for $G_1^\ast,\dots,G_K^\ast$. The key identity~\eqref{eqn:sum_of_squares} establishing the equivalence relies on two facts of the Euclidean space: (i) it is a vector space (i.e., vectors can be averaged in the linear sense); (ii) it is flat (i.e., zero curvature), both of which are unfortunately not true for the Wasserstein space $(\cP_2(\R^p), W_2)$ that endows the space $\cP_2(\R^p)$ of all probability distributions with finite second moments with the $2$-Wasserstein metric $W_2$~\citep{ambrosio2005gradient}. In particular, the 2-Wasserstein distance between two distributions $\mu$ and $\nu$ in $\cP_2(\bR^p)$ is defined as
\begin{equation}
    \label{eqn:kantorovich_problem}
    W_2^2(\mu, \nu) := \min_{\gamma} \Big\{ \int_{\bR^p \times \bR^p} \|x - y\|_2^2 \; \rd \gamma(x,y) \Big\},
\end{equation}
where minimization over $\gamma$ runs over all possible couplings with marginals $\mu$ and $\nu$. It is well-known that the Wasserstein space is a metric space (in fact a geodesic space) with non-negative curvature in the Alexandrov sense~\citep{Lott2008}.

{\bf Our contributions.} We summarize our main contributions as followings: (i) we provide evidence for pitfalls (irregularity and non-robustness) of barycenter-based Wasserstein $K$-means, both theoretically and empirically, and (ii) we generalize the distance-based formulation of $K$-means to the Wasserstein space and establish the exact recovery property of its SDP relaxation for clustering Gaussian measures under a separateness lower bound in the 2-Wasserstein distance.


{\bf Existing work.} Since the $K$-means clustering is a worst-case NP-hard problem~\citep{aloise2009np}, approximation algorithms have been extensively studied in literature including: Llyod's algorithm~\citep{Lloyd1982_TIT}, spectral methods~\citep{vanLuxburg2007_spectralclustering,Meila01learningsegmentation,NgJordanWeiss2001_NIPS}, semidefinite programming (SDP) relaxations~\citep{PengWei2007_SIAMJOPTIM}, non-convex methods via low-rank matrix factorization~\citep{BererMonteiro2003}. Theoretic guarantees of those methods are established for statistical models on Euclidean data~\citep{LuZhou2016,vonLuxburgBelkinBousquet2008_AoS,Vempala04aspectral,FeiChen2018,GiraudVerzelen2018,chen2021cutoff,pmlr-v151-zhuang22a}.

 The concept of clustering general measure-valued data is introduced by~\cite{Domazakis2019ClusteringMD}, where the authors proposed the centroid-based Wasserstein K-means algorithm. It replaced the Euclidean norm and sample means by the Wasserstein distance and barycenters respectively.  \cite{VerdinelliWasserman2019} proposed a modified Wasserstein distance for distribution clustering. And after that,~\cite{10.1214/21-EJS1834} proposed a method in Clustering of measures via mean measure quantization by first vectorizing the measures in a finite Euclidean space followed by an efficient clustering algorithm such as single-linkage clustering with $L_\infty$ distance. The vectorization methods could improve the computational efficiency but might not capture the properties of probability measures well compared to clustering algorithms based directly on Wasserstein space.

\section{Wasserstein $K$-means clustering methods}
\label{sec:wasserstein_Kmeans}

In this section, we generalize the Euclidean $K$-means to the Wasserstein space. Our starting point is to mimic the standard $K$-means methods for Euclidean data. Thus we may define two versions of the Wasserstein $K$-means clustering formulations: \emph{centroid-based} and \emph{distance-based}.
As we mentioned in Section~\ref{sec:intro}, when working with Wasserstein space $(\cP_2(\R^p), W_2)$, the corresponding centroid-based criterion~\eqref{eqn:euclidean_kmeans_center-based} and the distance-based criterion~\eqref{eqn:euclidean_kmeans_partition-based}, where the Euclidean metric $\|\cdot\|_2$ is replaced with the $2$-Wasserstein metric $W_2$, may lead to radically different clustering schemes. To begin with, we would like to argue that due to the irregularity and non-robustness of barycenters in the Wasserstein space, the centroid-based criterion may lead to unreasonable clustering schemes that lack physical interpretations and are sensitive to small data perturbations.

\subsection{Clustering based on barycenters}
\label{subsec:centroid-based_Wasserstein_Kmeans}

The centroid-based Wasserstein $K$-means for extending the Lloyd's algorithm into the Wasserstein space has been recently considered by~\cite{Domazakis2019ClusteringMD}. Specifically, it is an iterative algorithm proceeds as following. Given an initial centroid estimate $\nu_1^{(1)},\dots,\nu_K^{(1)}$, one first assigns each probability measure $\mu_1,\dots,\mu_n$ to its nearest centroid in the Wasserstein geometry at the $t$-th iteration according to the Voronoi diagram:
\begin{equation}
    \label{eqn:voronoi_diagram_wasserstein}
    G_k^{(t)} = \left\{ i \in [n] : W_2(\mu_i, \nu_k^{(t)}) \leq  W_2(\mu_i, \nu_j^{(t)}), \quad \forall j \in [K] \right\},
\end{equation}
and then update the centroid for each cluster
\begin{equation}
    \label{eqn:lloyd_update_wasserstein}
    \nu_k^{(t+1)} = \arg\min_{\nu \in \cP_2(\R^d)} {1\over|G_k^{(t)}|} \sum_{i \in G_k^{(t)}} W_2^2(\mu_i, \nu).
\end{equation}
Note that $\nu_k^{(t+1)}$ in~\eqref{eqn:lloyd_update_wasserstein} is referred as \emph{barycenter} of probability measures $\mu_i, i \in G_k^{(t)}$, a Wasserstein analog of the Euclidean average or mean~\citep{AguehCarlier2011}. We will also ex-changeably use barycenter-based $K$-means to mean the centroid-based K-means in the Wasserstein space.
Even though the Wasserstein barycenter is a natural notion of averaging probability measures, it may exhibit peculiar behaviours and fail to represent the within-cluster data points, partly due to the violation of the generalized parallelogram law~\eqref{eqn:sum_of_squares} (for non-flat spaces) if the Euclidean metric $\|\cdot\|_2$ is replaced with the $2$-Wasserstein metric $W_2$.

\smallskip
\begin{ex}[\bf Irregularity of Wasserstein barycenters]
\label{ex:wasserstein_barycenter_support}
Wasserstein barycenter has much less regularity than the sample mean in the Euclidean space~\citep{KIM2017640}. In particular,~\cite{SANTAMBROGIO2016152} constructed a simple example of two probability measures that are supported on line segments in $\bR^2$, whereas the support of their barycenter obtained as the displacement interpolation the two endpoint probability measures is not convex (cf. left plot in Figure~\ref{fig:wasserstein_barycenter_support}). In this example, the probability density $\mu_0$ and $\mu_1$ are supported on the line segments $L_0 = \{(s, a s) : s \in [-1, 1]\}$ and $L_1 = \{(s, -a s) : s \in [-1, 1]\}$ respectively. We choose $a \in (0, 1)$ to identify the orientation of $L_0$ and $L_1$ based on the $x$-axis. Moreover, we consider the linear density functions $\mu_0(s) = (1-s)/2$ and $\mu_1(s) = (1+s)/2$ for $s \in [-1, 1]$ supported on $L_0$ and $L_1$ respectively. Then the optimal transport map $T:=T_{\mu_0 \to \mu_1}$ from $\mu_0$ to $\mu_1$ is given by
\begin{equation}
    \label{eqn:OT_map_irregularity}
    T(x, a x) = \big( -1+\sqrt{4-(1-x)^2}, \ \ -a \cdot (-1+\sqrt{4-(1-x)^2}) \big),
\end{equation}
and the barycenter corresponds to the displacement interpolation $\mu_t = [(1-t)\text{id} + t T]_\sharp \mu_0$ at $t = 0.5$~\citep{MCCANN1997153}. For self-contained purpose, we give the proof of~\eqref{eqn:OT_map_irregularity} in Appendix~\ref{app:subsec:counterexample_irregularity}.
Fig.~\ref{fig:wasserstein_barycenter_support} on the left shows the support of barycenter $\mu_{0.5}$ is not convex (in fact part of an ellipse boundary) even though the supports of $\mu_0$ and $\mu_1$ are convex. This example shows that the barycenter functional is not geodesically convex in the Wasserstein space. As barycenters turn out to be essential in centroid-based Wasserstein $K$-means and irregularity of the barycenter may fail to represent the cluster (see more details in Example~\ref{ex:failure_barycenter_Wasserstein_Kmeans} and Remark~\ref{rem:further_insight} below), this counter-example is our motivation to seek alternative formulation. \qed

\begin{figure}[h!]
   \centering
   \subfigure{\includegraphics[trim={0.1cm 0.1cm 0.1cm 1cm},clip,scale=0.45]{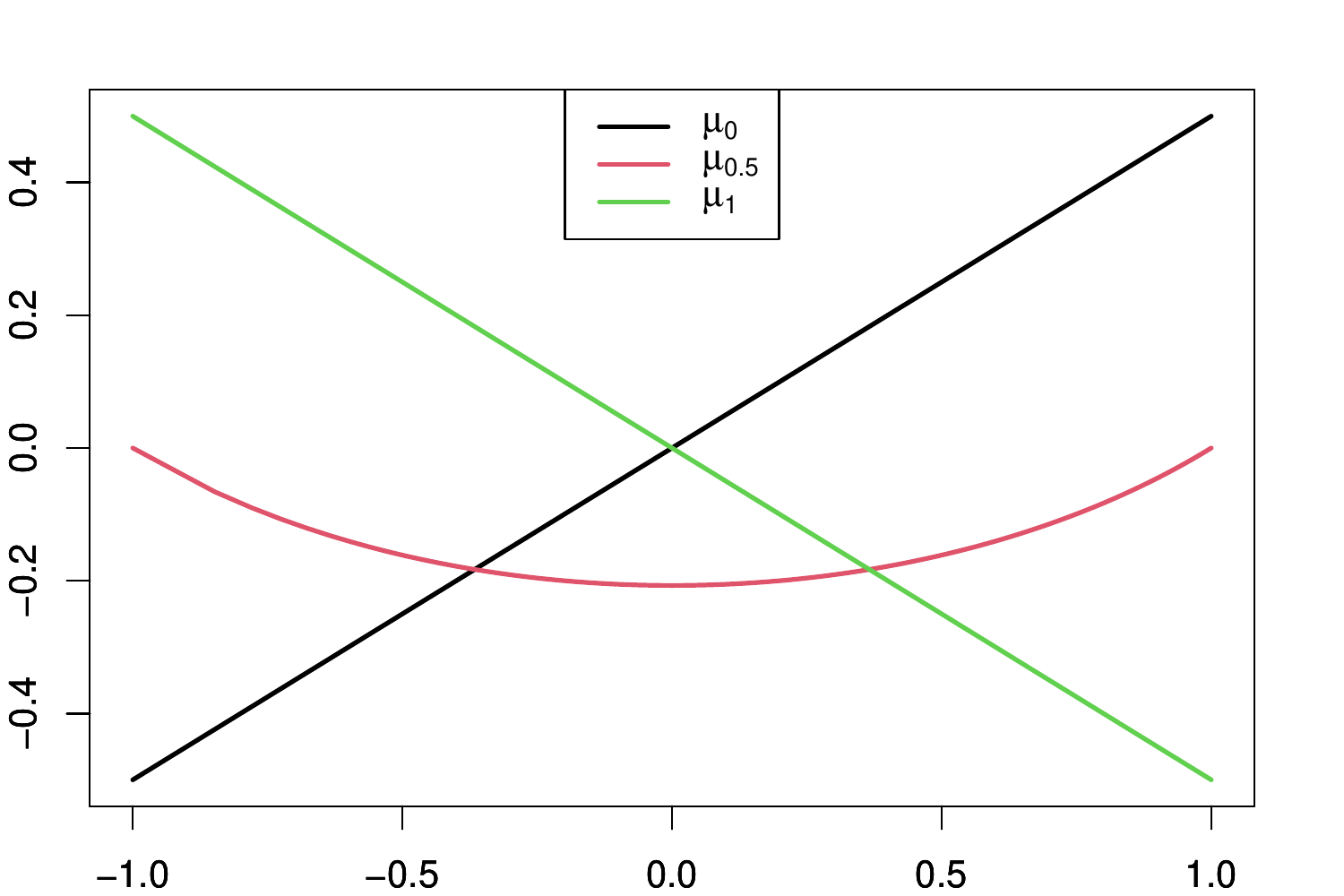}} 
    \subfigure{\includegraphics[trim={0.1cm 0.1cm 0.1cm 1cm},clip,scale=0.45]{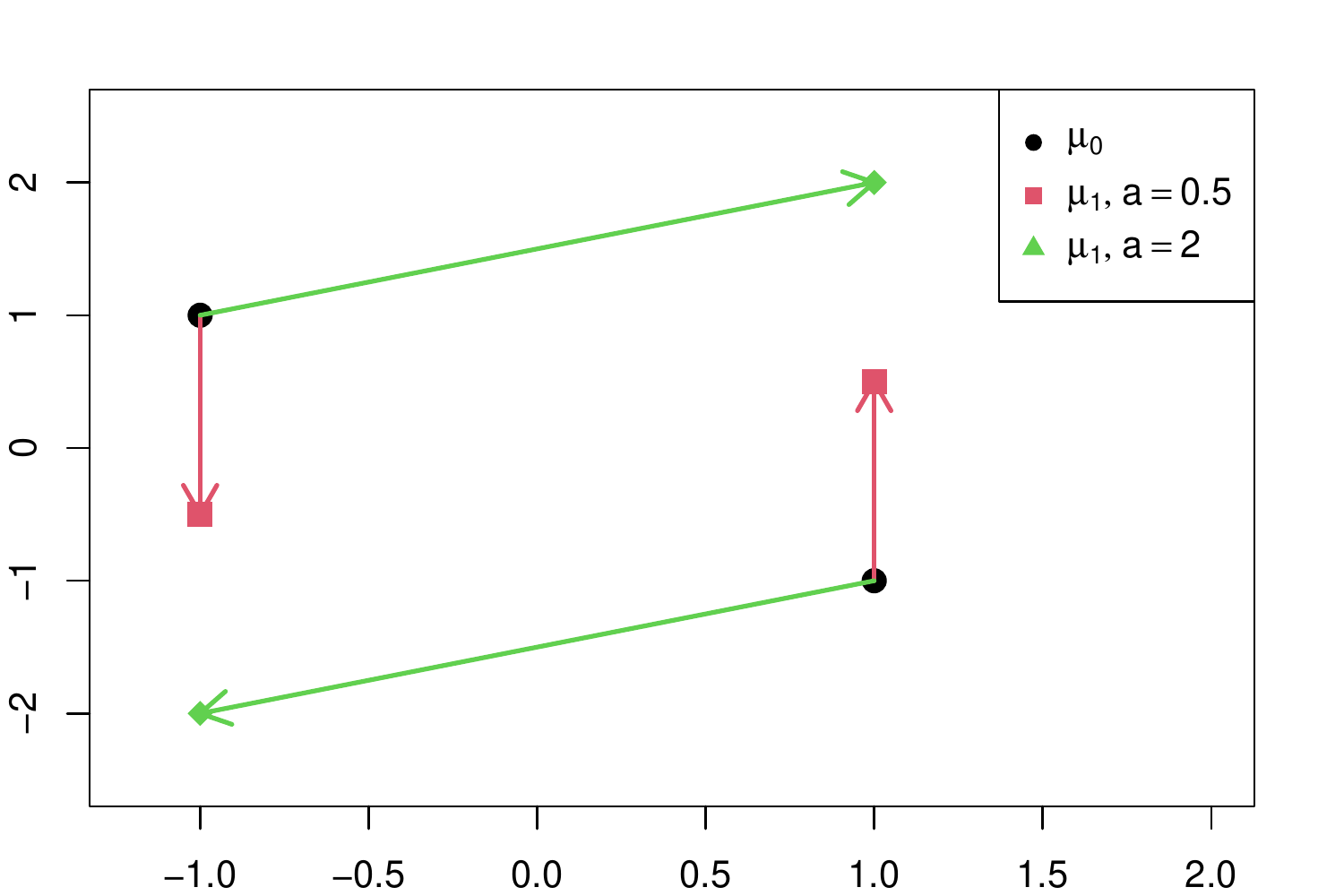}} 
   \caption{Left: support of the Wasserstein barycenter as the displacement interpolation between $\mu_0$ and $\mu_1$ at $t=0.5$ in Example~\ref{ex:wasserstein_barycenter_support}. Right: non-robustness of the optimal transport map (arrow lines) and Wasserstein barycenter w.r.t. small perturbation around $a=1$ for the target measure in Example~\ref{eg:non_robust}.}
   \label{fig:wasserstein_barycenter_support}
\end{figure}
\end{ex}

\smallskip
\begin{ex}[\bf Non-robustness of Wasserstein barycenters]\label{eg:non_robust}
Another unappealing feature of the Wasserstein barycenter is its sensitivity to data perturbation: a small (local) change in one contributing probability measure may lead to large (global) changes in the resulting barycenter. See Fig.~\ref{fig:wasserstein_barycenter_support} on the right for such an example. In this example, we take the source measure as $\mu_0 = 0.5 \, \delta_{(-1,1)} + 0.5 \,\delta_{(1,-1)}$ and the target measure as $\mu_1 = 0.5 \,\delta_{(-1,-a)} + 0.5 \,\delta_{(1,a)}$ for some $a > 0$. It is easy to see that the optimal transport map $T := T_{\mu_0 \to \mu_1}$ has a dichotomy behavior:
\begin{equation}
    T(-1, 1) = \left\{
    \begin{matrix}
    (-1, -a) & \text{if } 0 < a < 1 \\
    (1, a) & \text{if } a > 1 \\
    \end{matrix}
    \right. \quad \text{and} \quad
    T(1, -1) = \left\{
    \begin{matrix}
    (1, a) & \text{if } 0 < a < 1 \\
    (-1, -a) & \text{if } a > 1 \\
    \end{matrix}
    \right. .
\end{equation}
Thus the Wasserstein barycenter determined by the displacement interpolation $\mu_t = [(1-t)\text{id} + t T]_\sharp \mu_0$ is a discontinuous function at $a=1$. This non-robustness can be attribute to the discontinuity of the Wasserstein barycenter as a function of its input probability measures; in contrast, the Euclidean mean is a globally Lipchitz continuous function of its input points. \qed
\end{ex}


  

Because of these pitfalls of the Wasserstein barycenter shown in Examples~\ref{ex:wasserstein_barycenter_support} and~\ref{eg:non_robust}, the centroid-based Wasserstein $K$-means approach described at the beginning of this subsection may lead to unreasonable and unstable clustering schemes. In addition, an ill-conditioned configuration may significantly slow down the convergence of commonly used barycenter approximating algorithms such as iterative Bregman projections~\citep{benamou2015iterative}. Below, we give a concrete example of such phenomenon in the clustering context.

\smallskip
\begin{ex}[\bf Failure of centroid-based Wasserstein $K$-means]
\label{ex:failure_barycenter_Wasserstein_Kmeans}
In a nutshell, the failure in this example is due to the counter-intuitive phenomenon illustrated in the right panel of Fig.~\ref{fig:counterexample1}, where some distribution $\mu_3$ in the Wasserstein space may have larger $W_2$ distance to Wasserstein barycenter $\mu^\ast_1$ than every distribution $\mu_i$ ($i=1,2$) that together forms it. As a result of this strange configuration, even though $\mu_3$ is closer to $\mu_1$ and $\mu_2$ from the first cluster with barycenter $\mu^\ast_1$ than $\mu_4$ coming from a second cluster with barycenter $\mu_2^\ast$, it will be incorrectly assigned to the second cluster using the centroid-based criterion~\eqref{eqn:voronoi_diagram_wasserstein}, since $W_2(\mu_3,\mu^\ast_1) > W_2(\mu_3,\mu^\ast_2) > \max\big\{W_2(\mu_3,\mu_1),\, W_2(\mu_3,\mu_2)\big\}$. In contrast, for Euclidean spaces due to the following equivalent formulation of the generalized parallelogram law~\eqref{eqn:sum_of_squares},
\begin{align*}
    \sum_{i=1}^n \|X - X_i\|_2^2 = n \|X-\bar X\|_2^2 + \sum_{i=1}^n \|X_i-\bar X\|_2^2\geq n \|X-\bar X\|_2^2,\quad \mbox{for any }X\in\R^p,
\end{align*}
there is always some point $X_{i^\dagger}$ satisfying $\|X-X_{i^\dagger}\|_2 \geq \|X-\bar X\|_2$, that is, further away from $X$ than the mean $\bar X$; thereby excluding counter-intuitive phenomena as the one shown in Fig.~\ref{fig:counterexample1}.

\begin{figure}[h!] 
   \centering
      \subfigure{\includegraphics[trim={0.1cm 1cm 0.1cm 1cm},clip,scale=0.19]{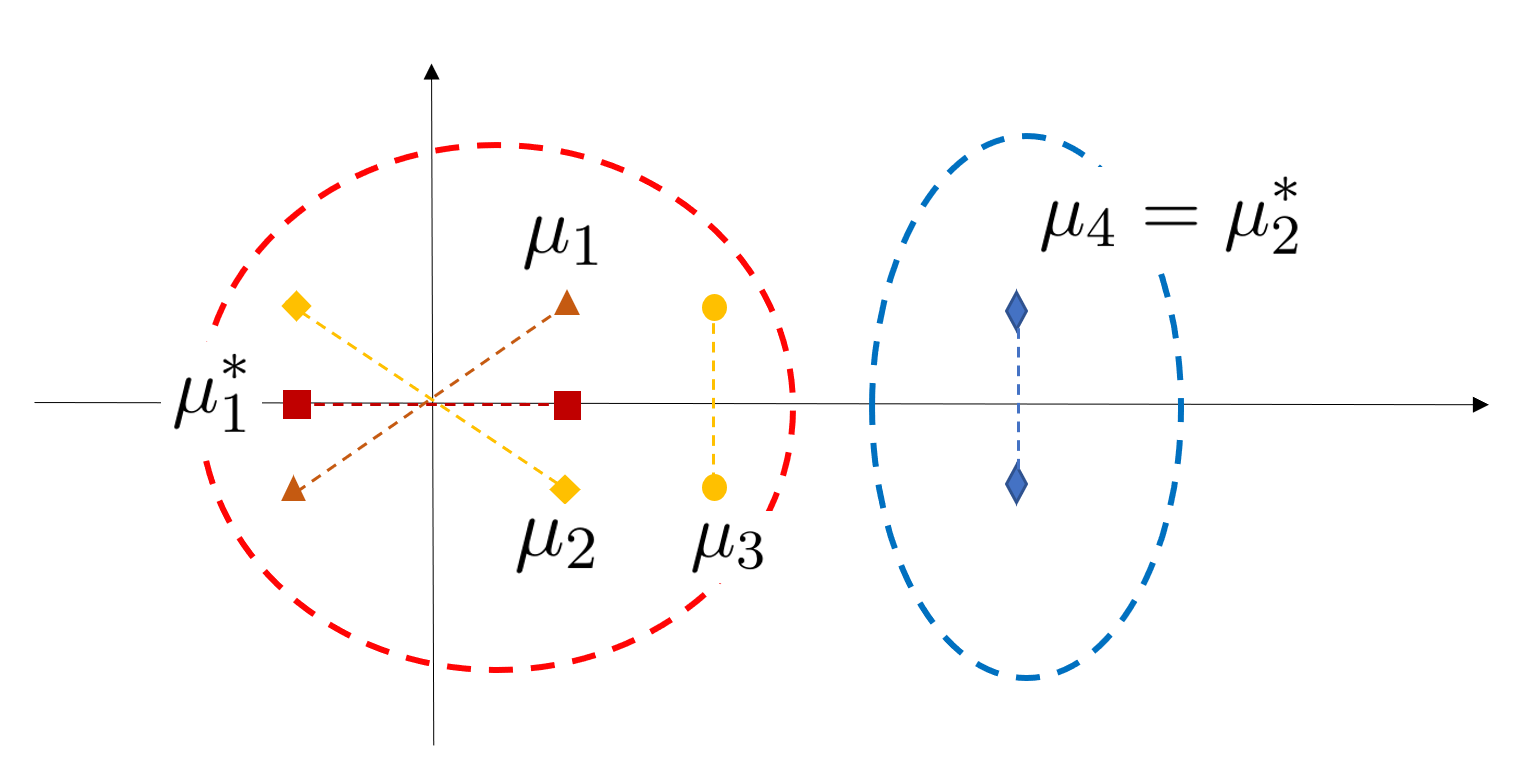}} 
      \subfigure{\includegraphics[trim={0.1cm 1cm 0.1cm 1cm},clip,scale=0.13]{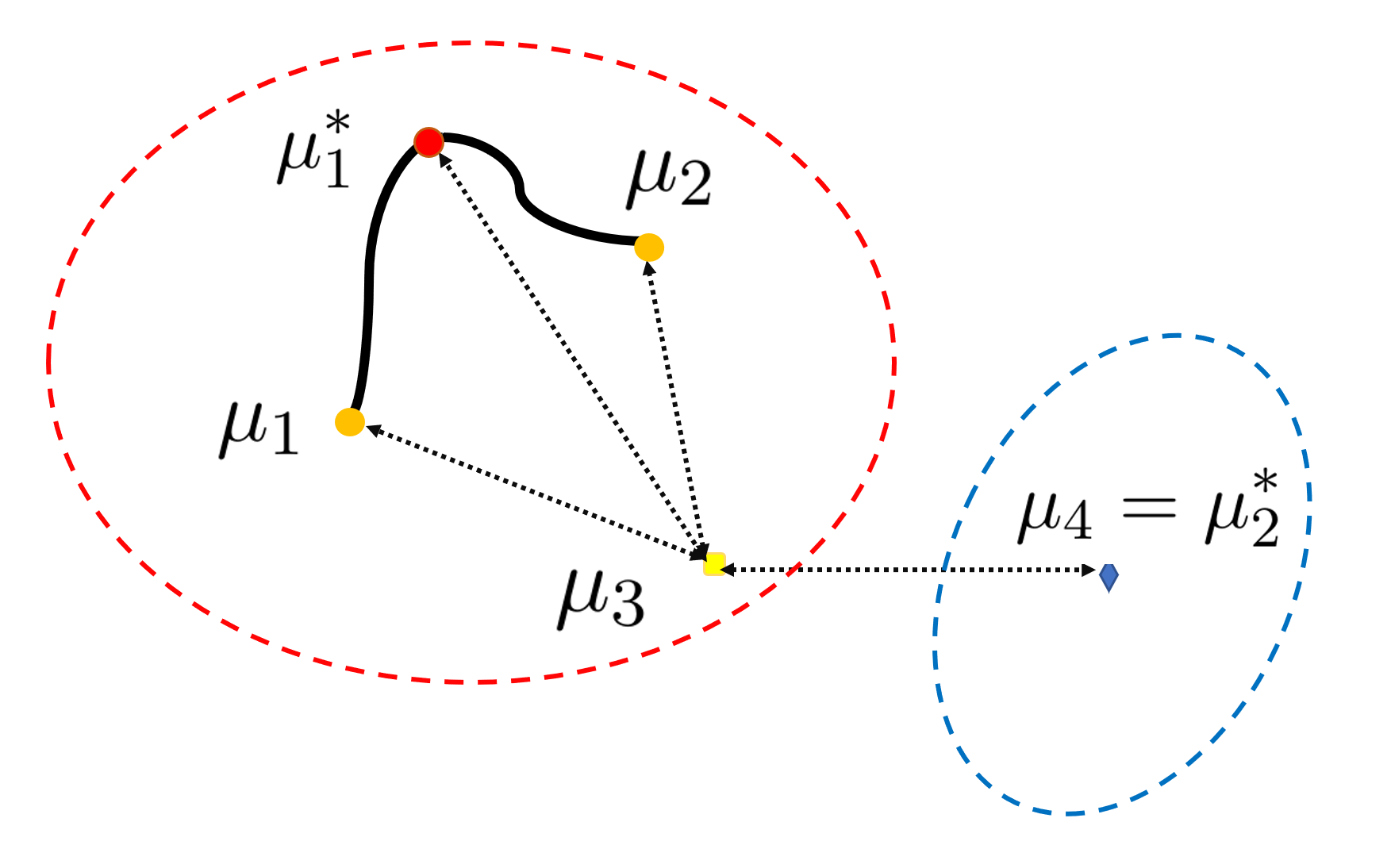}} \\[-3ex]
      \vspace{.1in}
   \caption{Left: visualization of Example~\ref{ex:failure_barycenter_Wasserstein_Kmeans} in $\mathbb{R}^2$ and Wasserstein space. Right: the black curve connecting $\mu_1$ and $\mu_2$ depicts the geodesic between them.}
   \label{fig:counterexample1}
   \vspace{-0.4cm}
\end{figure}

Concretely, the first cluster $G_1^\ast$ is shown in the left panel of Fig.~\ref{fig:counterexample1} highlighted by a red circle, consisting of $m$ copies of $(\mu_1,\mu_2 )$ pairs and one copy of $\mu_3$; the second cluster $G_2^\ast$ containing copies of $\mu_4$ is highlighted by a blue circle. Each distribution assigns equal probability mass to two points, where the two supporting points are connected by a dashed line for easy illustration. More specifically, we set
\begin{align*}
    \mu_1&=0.5\,\delta_{(x,y)}+0.5\,\delta_{(-x,-y)}, \quad \mu_2=0.5\,\delta_{(x,-y)}+0.5\,\delta_{(-x,y)}, \\
    \mu_3&=0.5\,\delta_{(x+\epsilon_1,y)}+0.5\,\delta_{(x+\epsilon_1,-y)}, \quad\mbox{and}\quad  \mu_4=0.5\,\delta_{(x+\epsilon_1+\epsilon_2,y)}+0.5\,\delta_{(x+\epsilon_1+\epsilon_2,-y)},
\end{align*}
where $\delta_{(x,y)}$ denotes the point mass measure at point $(x,y)$, and $(x,y,\epsilon_1,\epsilon_2)$ are positive constants. The property of this configuration can be summarized by the following lemma.

\smallskip
\begin{lem}[\bf Configuration characterization]\label{lem:counter_example}
If $(x,y,\epsilon_1,\epsilon_2)$ satisfies
\[
y^2<\min\{x^2, 0.25\,\Delta_{\epsilon_1,x} \} \quad\mbox{and}\quad \Delta_{\epsilon_1,x}<\epsilon_2^2<\Delta_{\epsilon_1,x}+y^2,
\]
where $\Delta_{\epsilon_1,x}:=\epsilon_1^2+2x^2+2x\epsilon_1$, then for all sufficiently large $m$ (number of copies of $\mu_1$ and $\mu_2$),
\[
W_2(\mu_3, \mu^\ast_2)<W_2(\mu_3, \mu^\ast_1) \qquad\mbox{and} \quad \underbrace{\max_{k=1,2}\max_{i,j\in G_k} W_2(\mu_i, \mu_j)}_{\mbox{\small largest within-cluster distance}} \! < \! \underbrace{\min_{i\in G_1, j\in G_2} W_2(\mu_i, \mu_j),}_{\mbox{\small least between-cluster distance}}
\]
where $\mu^\ast_k$ denotes the Wasserstein barycenter of cluster $G_k$ for $k=1,2$.
\end{lem} 
Note that the condition of Lemma~\ref{lem:counter_example} implies $y<x$. Therefore, the barycenter between $\mu_1$ and$ \mu_2$ is $\tilde{\mu}^\ast_1:=0.5\,\delta_{(x,0)}+0.5\,\delta_{(-x,0)}$ lying on the horizontal axis. By increasing $m$, the barycenter $\mu^\ast_1$ of cluster $G_1^\ast$ can be made arbitrarily close to $\tilde{\mu}^\ast$.  The second inequality in Lemma~\ref{lem:counter_example} shows that all within-cluster distances are strictly less than the between-cluster distances; therefore, clustering based on pairwise distances is able to correctly recover the cluster label of $\mu_3$. However, since $\mu_3$ is closer to the barycenter $\mu^\ast_2$ of cluster $G_2^\ast$ according to the first inequality in Lemma~\ref{lem:counter_example}, it will be mis-classified into $G_2^\ast$ using the centroid-based criterion. We emphasize that cluster positions in this example are generic and do exist in real data; see Remark~\ref{rem:further_insight} and Section~\ref{subsec:realdata} for further discussions on our experiment results on MNIST data. Moreover, similar to Example~\ref{eg:non_robust}, a small change in the orientation of distribution $\mu_1$ may completely alter the clustering membership of $\mu_3$ based on the centroid criterion. Specifically, if we slightly increase $x$ to make it exceed $y$, then the barycenter between $\mu_1$ and $\mu_2$ becomes $\bar{\mu}^\ast_1:=0.5\,\delta_{(0,y)}+0.5\,\delta_{(0,-y)}$ that lies on the vertical axis. Correspondingly, if based on centroids, then $\mu_3$ should be clustered into $G_1^\ast$ as it is closer to the barycenter $\mu_1^\ast$ of $G_1^\ast$ than the barycenter $\mu_2^\ast$ of $G_2^\ast$. Therefore, the centroid-based criterion can be unstable against data perturbations. In comparison, a pairwise distances based criterion always assigns $\mu_3$ into cluster $G_2^\ast$ no matter $x<y$ or $x>y$. \qed

\end{ex}




\subsection{Clustering based on pairwise distances}
Due to the irregularity and non-robustness of centroid-based Wasserstein $K$-means, we instead propose and advocate the use of distance-based Wasserstein $K$-means below, which extends the Euclidean distance-based $K$-means formulation~\eqref{eqn:euclidean_kmeans_partition-based} into the Wasserstein space,
\begin{equation}
    \label{eqn:wasserstein_kmeans_partition}
        \min_{G_1,\dots,G_K} \Big\{ \sum_{k=1}^K {1\over|G_k|} \sum_{i,j \in G_k} W_2^2(\mu_i, \mu_j) : \bigsqcup_{k=1}^K G_k = [n] \Big\}.
\end{equation}
Correspondingly, we can analogously design a greedy algorithm resembling the Wasserstein Lloyd's algorithm described in Section~\ref{subsec:centroid-based_Wasserstein_Kmeans} that solves the centroid-based Wasserstein $K$-means.
Specifically, the greedy algorithm proceeds in an iterative manner as following. Given an initial cluster membership estimate $G_1^{(1)},\dots,G_K^{(1)}$, one assigns each probability measure $\mu_1,\dots,\mu_n$ based on minimizing the averaged squared $W_2$ distances to all current members in every cluster, leading to an updated cluster membership estimate
\begin{equation}
    \label{eqn:cluster_mem_wasserstein}
    G_k^{(t+1)} = \bigg\{ i \in [n] : \frac{1}{|G_k^{(t)}|} \sum_{s\in G_k^{(t)}} W^2_2(\mu_i, \mu_s) \leq  \frac{1}{|G_j^{(t)}|} \sum_{s\in G_j^{(t)}} W^2_2(\mu_i, \mu_s), \quad \forall j \in [K] \bigg\}.
\end{equation}
We arbitrarily select among the least $W_2$ distance clusters in the case of a tie.
We highlight that the center-based and distance-based Wasserstein $K$-means formulations may not necessarily be equivalent to yield the same cluster labels (cf.~Example~\ref{ex:failure_barycenter_Wasserstein_Kmeans}).  Below, we shall give some example illustrating connections to the standard $K$-means clustering in the Euclidean space.

\smallskip
\begin{ex}[\bf Degenerate probability measures]
\label{ex:dirac_delta}
If the probability measures are Dirac at point $X_i \in \R^p$, i.e., $\mu_i = \delta_{X_i}$, then the Wasserstein $K$-means is the same as the standard $K$-means since $W_2(\mu_i, \mu_j) = \|X_i-X_j\|_2$.\qed
\end{ex}

\smallskip
\begin{ex}[\bf Gaussian measures]
\label{ex:gaussian_measures}
If $\mu_i = N(m_i, V_i)$ with positive-definite covariance matrices $\Sigma_i \succ 0$, then the squared $2$-Wasserstein distance can be expressed as the sum of the squared Euclidean distance on the mean vector and
\begin{equation}
\label{eqn:bures_metric_covmat}
    d^2(V_i, V_j) = \mbox{\rm Tr}\left[V_i+V_j-2\left(V_i^{1/2} V_j V_i^{1/2}\right)^{1/2}\right],
\end{equation}
the squared \emph{Bures distance} on the covariance matrix~\citep{BHATIA2019165}. Here, we use $V^{1/2}$ to denote the unique symmetric square root matrix of $V \succ 0$. That is,
\begin{equation}
    \label{eqn:wasserstein_bures_metric}
    W_2^2(\mu_i, \mu_j) = \|m_i-m_j\|_2^2 + d^2(V_i, V_j).
\end{equation}
Then the Wasserstein $K$-means, formulated either in~\eqref{eqn:voronoi_diagram_wasserstein} or~\eqref{eqn:wasserstein_kmeans_partition}, can be viewed as a \emph{covariance-adjusted} Euclidean $K$-means by taking account into the shape or orientation information in the (non-degenerate) Gaussian inputs.\qed
\end{ex}

\smallskip
\begin{ex}[\bf One-dimensional probability measures]
If $\mu_i$ are probability measures on $\R$ with cumulative distribution function (cdf) $F_i$, then the Wasserstein distance can be written in terms of the \emph{quantile transform}
\begin{equation}
    \label{eqn:quantile_transform}
    W_2^2(\mu_i, \mu_j) = \int_0^1 [F_i^-(u) - F_j^-(u)]^2 \; \rd u,
\end{equation}
where $F^-$ is the generalized inverse of the cdf $F$ on $[0,1]$ defined as $F^-(u) = \inf \{x \in \bR : F(x) > u\}$ (cf. Theorem 2.18~\citep{Villani2003_topics-in-ot}). Thus the one-dimensional probability measures in Wasserstein space can be isometrically embedded in a flat $L^2$ space, and we can bring back the equivalence of the Wasserstein and Euclidean $K$-means clustering methods.\qed
\end{ex}

\section{SDP relaxation and its theoretic guarantee}
\label{sec:SDP_relaxation}
Note that Wasserstein Lloyd's algorithm requires to use and compute the barycenter in~\eqref{eqn:voronoi_diagram_wasserstein} and~\eqref{eqn:lloyd_update_wasserstein} at each iteration, which can be computationally expensive when the domain dimension $d$ is large or the configuration is ill-conditioned (cf.~Example~\ref{eg:non_robust}). On the other hand, it is known that solving the distance-based $K$-means~\eqref{eqn:euclidean_kmeans_partition-based} is worst-case NP-hard for Euclidean data. Thus we expect solving the distance-based Wasserstein $K$-means~\eqref{eqn:wasserstein_kmeans_partition} is also computationally hard. A common way is to consider convex relaxations to approximate the solution of~\eqref{eqn:wasserstein_kmeans_partition}. 
It is known that certain SDP relaxation is information-theoretically tight for~\eqref{eqn:euclidean_kmeans_partition-based} when the data $X_1,\dots,X_n \in \bR^p$ are generated from a Gaussian mixture model with isotropic known variance~\citep{chen2021cutoff}. In this paper, we extend the idea into Wasserstein setting for solving~\eqref{eqn:wasserstein_kmeans_partition}. 

A typical SDP relaxation for Euclidean data uses pairwise inner products to construct an affinity matrix for clustering~\citep{PengWei2007_SIAMJOPTIM}; unfortunately, due to the non-flatness nature, a globally well-defined inner product does not exist for Wasserstein spaces with dimension higher than one. Therefore, we will derive a Wasserstein SDP relaxation to the combinatorial optimization problem~\eqref{eqn:euclidean_kmeans_partition-based} using the squared distance matrix $A_{n \times n} = \{a_{ij}\}$ with $a_{ij} = W_2^2(\mu_i, \mu_j)$. Concretely, we can one-to-one reparameterize any partition $(G_1,\dots,G_K)$ as a binary \emph{assignment matrix} $H = \{h_{ik}\} \in \{0,1\}^{n \times K}$ such that $h_{ik}=1$ if $i \in G_k$ and $h_{ik}=0$ otherwise. Then~\eqref{eqn:wasserstein_kmeans_partition} can be expressed as a nonlinear 0-1 integer program,
\begin{equation} 
    \min \Big\{ \langle A, H B H^\top \rangle : H \in \{0,1\}^{n \times K}, \, H \vone_K = \vone_n \Big\},
\end{equation}
where $\vone_n$ is the $n \times 1$ vector of all ones and $B = \diag(|G_1|^{-1},\dots,|G_K|^{-1})$. Changing of variable to the \emph{membership matrix} $Z = H B H^\top$, we note that $Z_{n \times n}$ is a symmetric positive semidefinite (psd) matrix $Z \succeq 0$ such that $\mbox{\rm Tr}(Z) = K, Z \vone_n = \vone_n$, and $Z \geq 0$ entrywise. Thus we obtain the SDP relaxation of~\eqref{eqn:wasserstein_kmeans_partition} by only preserving these convex constraints:
\begin{equation}
    \label{eqn:wasserstein_kmeans_SDP}
    \begin{gathered}
        \min_{Z \in \R^{n \times n}} \Big\{ \langle A, Z \rangle : Z^\top = Z, \, Z \succeq 0, \, \mbox{\rm Tr}(Z) = K, \, Z \vone_n = \vone_n, \, Z \geq 0 \Big\}.
    \end{gathered}
\end{equation}


To theoretically justify the SDP formulation~\eqref{eqn:wasserstein_kmeans_SDP} of Wasserstein $K$-means, we consider the scenario of clustering Gaussian distributions in Example~\ref{ex:gaussian_measures}, where the Wasserstein distance~\eqref{eqn:wasserstein_bures_metric} contains two separate components: the Euclidean distance on mean vector and the Bures distance~\eqref{eqn:bures_metric_covmat} on covariance matrix. Without loss of generality, we focus on mean-zero Gaussian distributions since optimal separation conditions for exact recovery based on the Euclidean mean component have been established in~\citep{chen2021cutoff}. Suppose we observe Gaussian distributions $\nu_i\sim N(0,V_i), \ i\in [n]$ from $K$ groups $G_1^*,\cdots,G_K^*$, where cluster $G_k^\ast$ contains $n_k$ members, and the covariance matrices have the following clustering structure: if $i\in G_k^*$, then
\begin{equation}
\label{eqn:stat_model}
    V_i=(I+tX_i)V^{(k)}(I+tX_i) \quad \text{with } X_1, \dots , X_n \overset{i.i.d.}{\sim}SymN(0, 1),
\end{equation}
where the psd matrix $V^{(k)}$ is the center of the $k$-th cluster, $SymN(0, 1)$ denotes the symmetric random matrix with i.i.d. standard normal entries, and $t$ is a small perturbation parameter such that  $(I+tX_i)$ is psd with high probability. For zero-mean Gaussian distributions, we have $W_2(N(0,V),N(0,U))=d(V,U)$ according to~\eqref{eqn:wasserstein_bures_metric}. Note that on the Riemannian manifold of psd matrices, the geodesic emanating from $V^{(k)}$ in the direction $X$ as a symmetric matrix can be linearized by $V = (I+tX)V^{(k)}(I+tX)$ in a small neighborhood of $t$, thus motivating the parameterization of our statistical model in~\eqref{eqn:stat_model}. The next theorem gives a separation lower bound to ensure exact recovery of the clustering labels for Gaussian distributions. 

\begin{thm}[\bf Exact recovery for clustering Gaussians]
\label{thm:separation_upper_bound_gaussians}
Let  $\Delta^2:=\min_{k\neq l} d^2(V^{(k)},V^{(l)})$ denote the minimal pairwise separation among clusters, $\bar{n}:=\max_{k\in[K]}n_k$ (and $\underline{n}:=\min_{k\in[K]}n_k$) the maximum (minimum) cluster size, and $m:=\min_{k\ne l}\frac{2n_kn_l}{n_k+n_l}$ the minimal pairwise harmonic mean of cluster sizes. Suppose the covariance matrix $V_i$ of Gaussian distribution $\nu_i=N(0, V_i)$ is independently drawn from model~\eqref{eqn:stat_model} for $i=1,2,\ldots,n$. Let $\beta\in(0,1)$. If the separation $\Delta^2$ satisfies
\begin{equation}
    \Delta^2 >  \bar{\Delta}^2:\,=\frac{C_1 t^2} {\min\{(1-\beta)^2,\beta^2\}}\, \mathcal{V}\, p^2\log n,
\end{equation}
then the SDP~\eqref{eqn:wasserstein_kmeans_SDP} achieves exact recovery with probability at least $1-C_2 n^{-1}$, provided that
$$ \underline{n} \ge C_3 \log^2 n , \ \ t\le C_4\sqrt{\log n}/\big[(p+\log \bar{n})  \mathcal{V}^{1/2}T_v^{1/2}\big], \ \ n/m\le C_5\log n,$$
where  $\mathcal{V}=\max_k \left\|V^{(k)} \right\|_{\mbox{\scriptsize \rm op}}$, $T_v=\max_k \mbox{\rm Tr}\big[\big(V^{(k)}\big)^{-1}\big]$, and $C_i,i=1,2,3,4,5$ are constants.
\end{thm}

\begin{rem}[{\bf Further insight on pitfalls of barycenter-based Wasserstein $K$-means}]
\label{rem:further_insight}
Theorem~\ref{thm:separation_upper_bound_gaussians} suggests that different from Euclidean data, distributions after centering can be clustered if scales and rotation angles vary (i.e., covariance-adjusted). We further illustrate the rotation and scale effects on the MNIST data that may mislead the centroid-based Wasserstein $K$-means, thus providing a real data support for Example~\ref{ex:failure_barycenter_Wasserstein_Kmeans}. Here we randomly sample two unbalanced clusters with 200 numbers of "0" and 100 numbers of "5". Fig.~\ref{fig:kmforMNIST} shows the clustering results for the centroid-based Wasserstein $K$-means and its \emph{oracle} version where we replace the estimated barycenters $\mu_1, \mu_2$ with the true barycenters $\mu_1^*, \mu_2^*$ computed on the true labels. Comparing the Wasserstein distances $W_2(\mu_0, \mu_1^*)$ and $W_2(\mu_0, \mu_2^*)$, we see that the image $\mu_0$ (containing digit "0") is closer to $\mu_2^*$ (true barycenter of digit "5") and thus it cannot be classified correctly based on the nearest true barycenter (cf. Fig.~\ref{fig:kmforMNIST} on the left). Moreover, Wasserstein $K$-means based on estimated barycenters $\mu_1, \mu_2$ yields two clusters of mixed "0" and "5". In both cases, the misclassification error is characterized by grouping similar degrees of angle and/or stretch. Since there are two highly unbalanced clusters of distributions, Wasserstein $K$-means is likely to enforce larger cluster to separate into two clusters and absorb those around centers (cf. Fig.~\ref{fig:kmforMNIST} on the right), leading to larger classification errors. We shall see that in Section~\ref{subsec:realdata} the distance-based Wasserstein $K$-means and its SDP relaxation have much smaller classification error rate on MNIST for the reason that we explained in Example~\ref{ex:failure_barycenter_Wasserstein_Kmeans} (cf. Lemma~\ref{lem:counter_example}). \qed
\end{rem}

\begin{figure}[t] 
   \centering
    \includegraphics[trim={0cm 0cm 0cm 1cm},clip,scale=0.29]{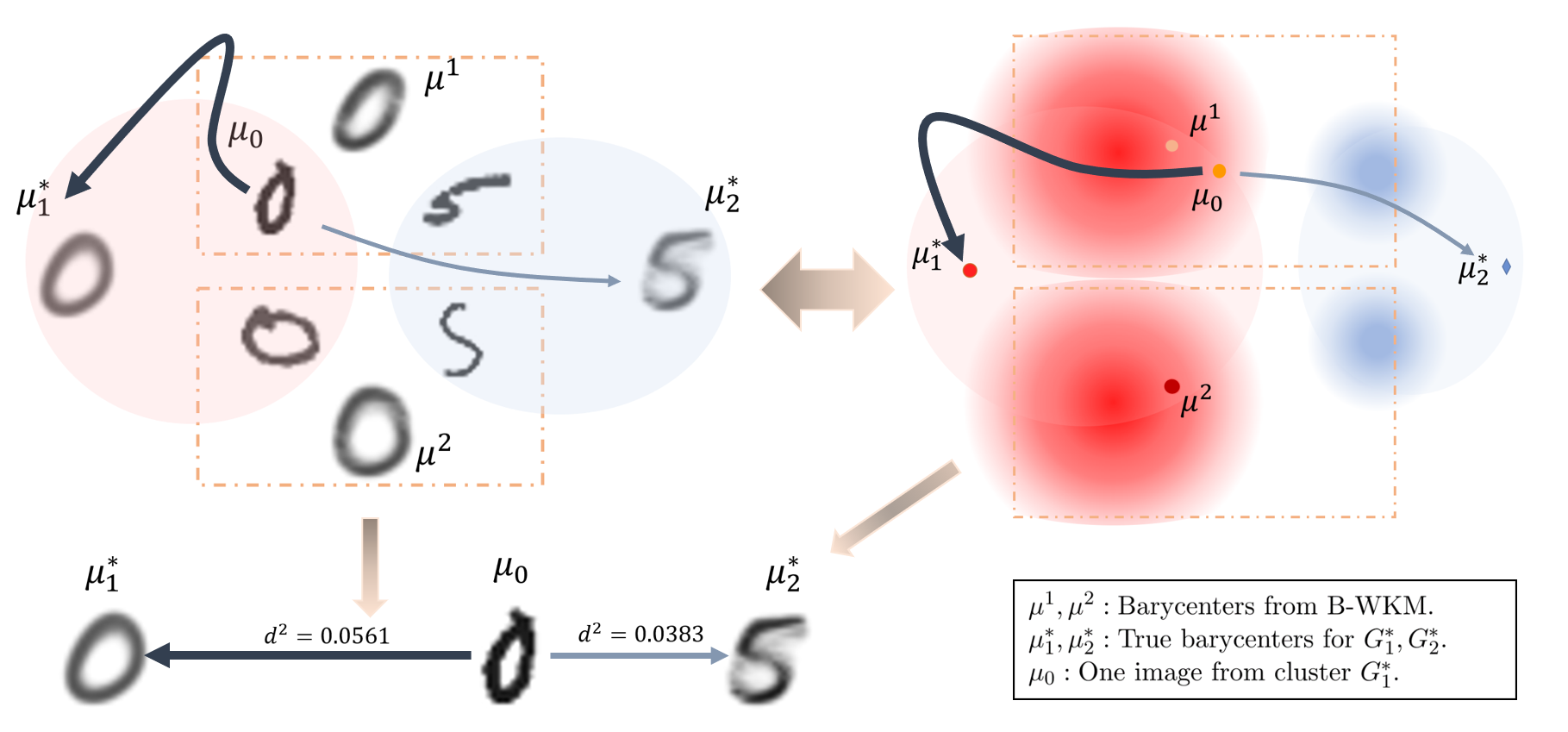}
    \caption{Visualization of misclassification for the barycenter-based Wasserstein $K$-means (B-WKM) on a randomly sampled subset from MNIST ($200$ digit "0" and $100$ digit "5"). The plot at the bottom is a example of misclassified image. The right plot is the abstraction of the images in the Wasserstein space. The color depth indicates the frequency of the distributions. Red and blue colors stand for distributions belong to true clusters "0" and "5".}
    \label{fig:kmforMNIST}%
    \vspace{-0.4cm}
\end{figure}

\section{Experiments}
\label{sec:sim}



\subsection{Counter-example in Example~\ref{ex:failure_barycenter_Wasserstein_Kmeans} revisited}
\label{subsec:counterexample}
Our first experiment is to back up the claim about the failure of centroid-based Wasserstein $K$-means in Example~\ref{ex:failure_barycenter_Wasserstein_Kmeans} through simulations. Instead of using point mass measures that may results in instability for computing the barycenters, we use Gaussian distributions with small variance as a smoothed version.  We consider $K=2$, where cluster $G_1^*$ consists of $m_1$ many copies of $(\mu_1,\mu_2) $ pairs and $m_2$ many $\mu_3$, and cluster $G_2^*$ consists of $m_3$ many copies of $\mu_4.$ We choose $\mu_i$ as the following two-dimensional mixture of Gaussian distributions $\mu_i=0.5\, N(a_{i,1},\Sigma_{i,1})+0.5\,N(a_{i,2},\Sigma_{i,2})$ for $i=1,2,3,4$. Due to the space limit, detailed simulation setups and parameters are given in Appendix~\ref{app:sub:sim_setup}. From Table~\ref{tab:results_conterexample2}, we can observe that Wasserstein SDP has achieved exact recovery for all cases while barycenter-based Wasserstein $K$-means has only around $40\%$ exact recovery rate among all repetitions. In addition, Wasserstein SDP is more stable than distance-based Wasserstein $K$-means. Denote $\Delta_k:=W^2(\mu_3, \mu^\ast_k)$ as the squared distance between $\mu_3$ and $\mu_k^\ast$ for $k=1,2$, where $\mu_k^*$ is the barycenter of $G_k^*$. Let $\Delta_\ast:=\max_{k=1,2}\max_{i,j\in G_k} W_2(\mu_i, \mu_j)$ and $\Delta^\ast:=\min_{i\in G_1, j\in G_2} W^2(\mu_i, \mu_j)$ be the maximum within-cluster distance and the minimum between-cluster distance respectively. From Table~\ref{tab:results_conterexample1} in the Appendix, we can observe that $\Delta_\ast< \Delta^\ast$, from which we can expect Wasserstein SDP to correctly cluster all data points in the Wassertein space. Moreover, Table~\ref{tab:results_conterexample2} shows that about $25\%$ times that the distributions (as $\mu_3$) in $G_1^\ast$ satisfy $\Delta_1>\Delta_2$, implying those $\mu_3$ to be likely assigned to the wrong cluster, which is consistent with Example~\ref{ex:failure_barycenter_Wasserstein_Kmeans}. The experiment results also show that any copy of $\mu_3$ is misclassified whenever exact recovery fails for B-WKM, which means the misclassified rate for $\mu_3$ equals to $(1-\gamma)$, where $\gamma$ is the exact recovery rate for B-WKM shown in Table~\ref{tab:results_conterexample2}. Table~\ref{tab:results_conterexample12} in the appendix further reports the run time comparison, from which we see that distance-based approaches are more computationally efficient than the barycenter-based one in our settings.


\begin{table}[h]
\caption{Exact recovery rates and frequency of $\Delta_1 >\Delta_2$ for B-WKM among total $50$ repetitions in the counter example. W-SDP: Wasserstein SDP, D-WKM: Distance-based Wasserstein $K$-means, B-WKM: Barycenter-based Wasserstein $K$-means. $n$: total number of distributions.} 
\vspace{-0.2cm}
\label{tab:results_conterexample2}
\begin{center}
\begin{tabular}{ccccc}
\toprule
$n$  &   W-SDP &  D-WKM & B-WKM & Frequency of $\Delta_1 >\Delta_2$\\
\hline
101 & 1.00 & 0.82 & 0.40 & 0.32 \\
202 & 1.00 & 0.84 & 0.34 & 0.26\\
303 & 1.00 & 0.72 & 0.46 & 0.20  \\
 \bottomrule
\end{tabular}
\vspace{-0.5cm}
\end{center}
\end{table}

\subsection{Gaussian distributions}
\label{subsec:simulation}


Next, we simulate random Gaussian measures from model~\eqref{eqn:stat_model} with $K=4$ and all cluster size equal. We set the centers of each cluster of Gaussians such that all pairwise distances among the barycenters are all equal, i.e., $W_2^2(N(0,V^{(k_1)}),N(0,V^{(k_2)})) \equiv D$ for all $k_1,k_2\in \{1,2,3,4\}$ with $\mathcal{V}=\max_k \|V^{(k)} \|_{\mbox{\scriptsize \rm op}}\in [4.5,5.5].$ We fix the dimension $p=10$ and vary the sample size $n=200,400,600$. And we set the perturbation parameter $t=10^{-3}$ on the covariance matrix. The simulation results are reported over $100$ times in each setting.  Fig.~\ref{fig:comparison for different n} shows the misclassification rate (log-scale) versus the squared distance $D$ between centers. We observe that when the distance between centers of clusters are larger than certain threshold (squared distance $D>10^{-3}$ in this case), then Wasserstein SDP can achieve exact recovery for different $n,$ while the misclassification rate for the two Wasserstein $K$-means are stably around $10\%$. And when the distance between centers of clusters are relatively small, the two Wasserstein $K$-means and SDP behave similarly.


\begin{figure}[h!] 
   \centering
     \subfigure{\includegraphics[trim={0.0cm 0cm 0.0cm 0.0cm},clip,scale=0.25]{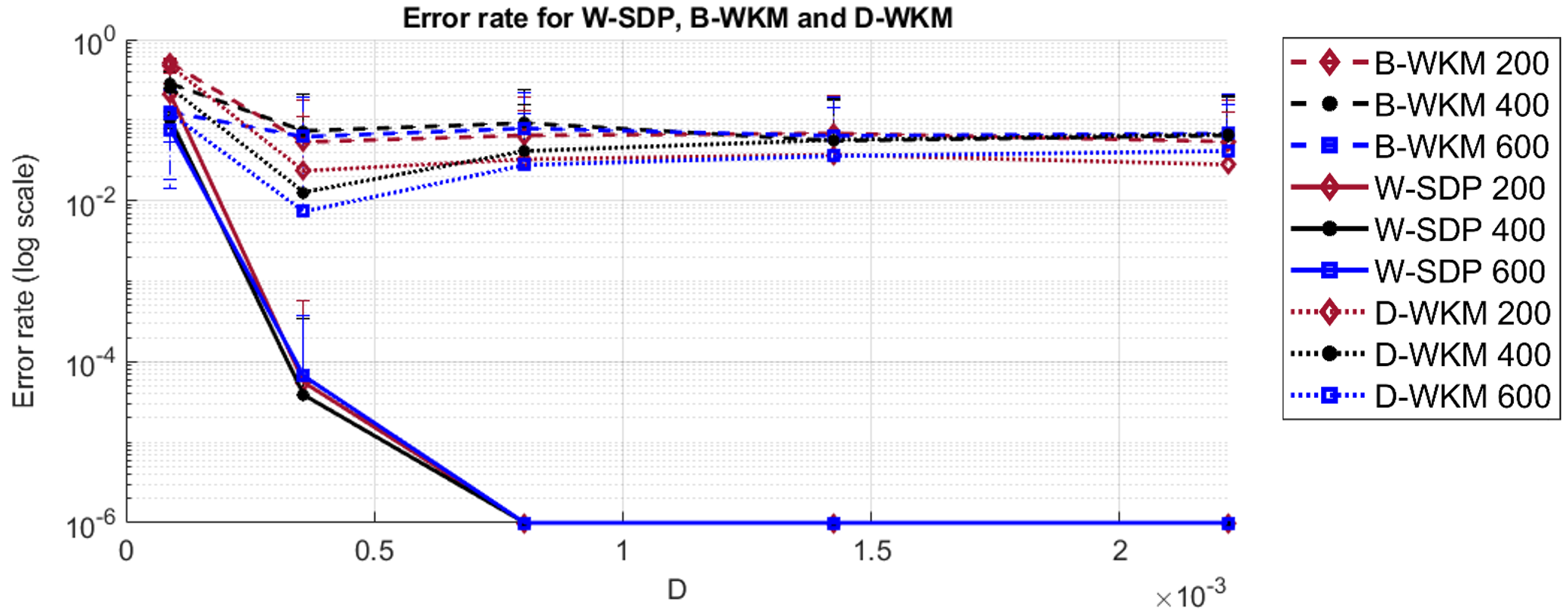}} 
   \caption{Mis-classification error versus squared distance $D$ from Wasserstein SDP (W-SDP) and barycenter/distance-based Wasserstein $K$-means (B-WKM and D-WKM) for clustering Gaussians under $n\in\{200,400,600\}$. Due to the log-scale, $10^{-6}$ corresponds to exact recovery.}
   \label{fig:comparison for different n}
   \vspace{-0.2cm}
\end{figure}

\subsection{Real-data applications}
\label{subsec:realdata}
We consider three benchmark image datasets: MNIST, Fashion-MNIST and USPS handwriting digits. Due to complexity issues, we consider subsets of the whole datasets and randomly choose fixed number of images from each clusters based on $10$ replicates for each cases. Here we used Sinkhorn divergence to calculate the Wasserstein distance and the Bregman projection with $100$ iterations for computing the barycenters, which is efficient and stable for non-degenerate case in practice. For both Wasserstein $K$-means methods, we use the initialization method in analogue to the $K$-means++ for Euclidean data, i.e., the first cluster barycenter is chosen uniformly at random as one of the distributions, after which each subsequent cluster barycenter is chosen from the remaining distributions with probability proportional to its squared Sinkhorn divergence from the distribution's closest existing cluster barycenter. Codes using MATLAB and Python implementing W-SDP and D-WKM are available at: \url{https://github.com/Yubo02/Wasserstein-K-means-for-clustering-probability-distributions}.

For MNIST dataset, we choose and fix two clusters: $G_1^\ast$ containing the number "0" and $G_2^\ast$ containing the number "5" for two cases. (1) In Case 1, we randomly draw $200$ number "0" and $100$ number of "5" for each repetition. (2) In Case 2, we double the number and randomly draw $400$ number "0" and $200$ number of "5" instead. For Fashion-MNIST and USPS handwriting digits dataset, we consider $K=3$: $G_1^\ast$ containing the "T-shirt/top" (or the number "0" for USPS handwriting), $G_2^\ast$ containing the "Trouser" (or the number "5" for USPS handwriting) and $G_3^\ast$ containing the "Dress" (or the number "7" for USPS handwriting). The cluster sizes are unbalanced where we randomly choose $200, 100$ and $100$ number from $G_1^\ast,G_2^\ast$ and $G_3^\ast$ respectively. The error rates are shown in Table~\ref{tab:Realdata}. Detailed setups and more results about $F_1$ scores (analogous to error rates) as well as time costs are placed at Appendix~\ref{app:real_setup} due to space limit.

From Table~\ref{tab:Realdata}, we can observe that the performances for Wasserstein SDP (W-SDP) and distance-based Wasserstein $K$-means (D-WKM) are better compared with barycenter-based Wasserstein $K$-means (B-WKM) for all cases. And the results from Table~\ref{tab:Realdata_f1} in Appendix~\ref{app:real_setup} by using $F_1$ score are consistent with the results from Table~\ref{tab:Realdata}. The original $K$-means method behaves similar as barycenter-based Wasserstein $K$-means in some cases and behaves less preferable for cases such as our experiment on USPS handwriting digits. In particular, the visualization of the clustering results for case 1 has been shown in Fig.~\ref{fig:kmforMNIST}. From this figure we can find that the classification criterion for B-WKM will end up with the closeness to certain shape of "0", which is characterized by certain angle or the degree of stretch. And this will lead to the high misclassification error for barycenter-based or centroid-based Wasserstein $K$-means. 

\begin{table}[h]
\caption{Error rate (SD) for clustering three benchmark datasets: MNIST, Fashion-MNIST and USPS handwriting digits. MNIST$_1$ (MNIST$_2$) refers to the results of Case 1 (Case 2) for MNIST dataset.}
\vspace{-0.2cm}
\label{tab:Realdata}
\begin{center}
\begin{tabular}{ccccc} 
\toprule
   & W-SDP & D-WKM  & B-WKM  & KM\\
\hline
MNIST$_1$  & 0.235 (0.045) & 0.156 (0.057) & 0.310 (0.069) & 0.295 (0.066)\\
MNIST$_2$  & 0.279 (0.050) & 0.185 (0.097) & 0.324 (0.032) & 0.362 (0.033)\\
Fashion-MNIST  & 0.082 (0.020) & 0.056 (0.014) & 0.141 (0.059) & 0.138 (0.099)\\
USPS handwriting & 0.206 (0.020)& 0.159 (0.061) & 0.240 (0.045) & 0.284 (0.025)\\
 \bottomrule
\end{tabular}
\vspace{-0.2cm}
\end{center}
\end{table}

\section{Discussion}
\label{sec:discussion}

In this paper, we observed and analyzed the peculiar behaviors of Wasserstein barycenters and their results in clustering probability distributions. After that, we proposed the distance-based K-means approach (D-WKW) and its semidefinite program  relaxation (W-SDP) by showing the exact recovery results for Gaussians theoretically and numerically. For several real benchmark datasets, we showed results where D-WKM and W-SDP could outperform  barycenter-based K-means approach (B-WKM). And we focused on one unbalanced case of two clusters from MNIST and analyzed its behavior through visualization of misclassification for the barycenter-based Wasserstein K-means. The corresponding time costs for B-WKM, D-WKM and W-SDP for benchmark datasets suggest that the scalability for them and especially our approaches could be serious when sample size is large. One of our future goals would be addressing the corresponding computational complexity issues for real data. Another goal is to find out more conclusive results where our approaches are preferable within or out of the realm of unbalanced clusters for real datasets.

\begin{ack}
Xiaohui Chen was partially supported by NSF CAREER grant DMS-1752614. Yun Yang was partially supported by NSF grant DMS-2210717.
\end{ack}

\newpage
\bibliographystyle{plainnat}
\bibliography{refs} 

\begin{thebibliography}{43}
\providecommand{\natexlab}[1]{#1}
\providecommand{\url}[1]{\texttt{#1}}
\expandafter\ifx\csname urlstyle\endcsname\relax
  \providecommand{\doi}[1]{doi: #1}\else
  \providecommand{\doi}{doi: \begingroup \urlstyle{rm}\Url}\fi

\bibitem[Adamczak(2008)]{10.1214/EJP.v13-521}
Radoslaw Adamczak.
\newblock {A tail inequality for suprema of unbounded empirical processes with
  applications to Markov chains}.
\newblock \emph{Electronic Journal of Probability}, 13\penalty0
  (none):\penalty0 1000 -- 1034, 2008.
\newblock \doi{10.1214/EJP.v13-521}.
\newblock URL \url{https://doi.org/10.1214/EJP.v13-521}.

\bibitem[Agueh and Carlier(2011)]{AguehCarlier2011}
Martial Agueh and Guillaume Carlier.
\newblock Barycenters in the wasserstein space.
\newblock \emph{SIAM J. Math. Anal.}, 43\penalty0 (2):\penalty0 904--924, 2011.

\bibitem[Aloise et~al.(2009)Aloise, Deshpande, Hansen, and Popat]{aloise2009np}
Daniel Aloise, Amit Deshpande, Pierre Hansen, and Preyas Popat.
\newblock Np-hardness of euclidean sum-of-squares clustering.
\newblock \emph{Machine learning}, 75\penalty0 (2):\penalty0 245--248, 2009.

\bibitem[Ambrosio et~al.(2005)Ambrosio, Gigli, and
  Savar{\'e}]{ambrosio2005gradient}
Luigi Ambrosio, Nicola Gigli, and Giuseppe Savar{\'e}.
\newblock \emph{Gradient flows: in metric spaces and in the space of
  probability measures}.
\newblock Springer Science \& Business Media, 2005.

\bibitem[Benamou et~al.(2015)Benamou, Carlier, Cuturi, Nenna, and
  Peyr{\'e}]{benamou2015iterative}
Jean-David Benamou, Guillaume Carlier, Marco Cuturi, Luca Nenna, and Gabriel
  Peyr{\'e}.
\newblock Iterative bregman projections for regularized transportation
  problems.
\newblock \emph{SIAM Journal on Scientific Computing}, 37\penalty0
  (2):\penalty0 A1111--A1138, 2015.

\bibitem[Bhatia et~al.(2019)Bhatia, Jain, and Lim]{BHATIA2019165}
Rajendra Bhatia, Tanvi Jain, and Yongdo Lim.
\newblock On the bures--wasserstein distance between positive definite
  matrices.
\newblock \emph{Expositiones Mathematicae}, 37\penalty0 (2):\penalty0 165--191,
  2019.
\newblock ISSN 0723-0869.
\newblock \doi{https://doi.org/10.1016/j.exmath.2018.01.002}.
\newblock URL
  \url{https://www.sciencedirect.com/science/article/pii/S0723086918300021}.

\bibitem[Bigot et~al.(2017)Bigot, Gouet, Klein, and L{\'o}pez]{Bigot2017_GPCA}
J{\'e}r{\'e}mie Bigot, Ra{\'u}l Gouet, Thierry Klein, and Alfredo L{\'o}pez.
\newblock {Geodesic PCA in the Wasserstein space by convex PCA}.
\newblock \emph{Annales de l'Institut Henri Poincar{\'e}, Probabilit{\'e}s et
  Statistiques}, 53\penalty0 (1):\penalty0 1 -- 26, 2017.
\newblock \doi{10.1214/15-AIHP706}.
\newblock URL \url{https://doi.org/10.1214/15-AIHP706}.

\bibitem[Brenier(1991)]{Brenier1991}
Yann Brenier.
\newblock Polar factorization and monotone rearrangement of vector-valued
  functions.
\newblock \emph{Communications on Pure and Applied Mathematics}, 44\penalty0
  (4):\penalty0 375--417, 1991.
\newblock \doi{https://doi.org/10.1002/cpa.3160440402}.
\newblock URL
  \url{https://onlinelibrary.wiley.com/doi/abs/10.1002/cpa.3160440402}.

\bibitem[Burer and Monteiro(2003)]{BererMonteiro2003}
Samuel Burer and Renato D.~C. Monteiro.
\newblock A nonlinear programming algorithm for solving semidefinite programs
  via low-rank factorization.
\newblock \emph{Mathematical Programming}, 95\penalty0 (2):\penalty0 329--357,
  2003.
\newblock \doi{10.1007/s10107-002-0352-8}.
\newblock URL \url{https://doi.org/10.1007/s10107-002-0352-8}.

\bibitem[Cazelles et~al.(2018)Cazelles, Seguy, Bigot, Cuturi, and
  Papadakis]{CazellesSeguyBigotCuturiPapadakis2018}
Elsa Cazelles, Vivien Seguy, J{\'e}r{\'e}mie Bigot, Marco Cuturi, and Nicolas
  Papadakis.
\newblock Geodesic pca versus log-pca of histograms in the wasserstein space.
\newblock \emph{SIAM Journal on Scientific Computing}, 40\penalty0
  (2):\penalty0 B429--B456, 2018.
\newblock \doi{10.1137/17M1143459}.
\newblock URL \url{https://doi.org/10.1137/17M1143459}.

\bibitem[Chazal et~al.(2021)Chazal, Levrard, and Royer]{10.1214/21-EJS1834}
Fr{\'e}d{\'e}ric Chazal, Cl{\'e}ment Levrard, and Martin Royer.
\newblock {Clustering of measures via mean measure quantization}.
\newblock \emph{Electronic Journal of Statistics}, 15\penalty0 (1):\penalty0
  2060 -- 2104, 2021.
\newblock \doi{10.1214/21-EJS1834}.
\newblock URL \url{https://doi.org/10.1214/21-EJS1834}.

\bibitem[Chen and Yang(2021)]{chen2021cutoff}
Xiaohui Chen and Yun Yang.
\newblock Cutoff for exact recovery of gaussian mixture models.
\newblock \emph{IEEE Transactions on Information Theory}, 67\penalty0
  (6):\penalty0 4223--4238, 2021.

\bibitem[Chen et~al.(2021)Chen, Lin, and M{\"u}ller]{ChenLinMuller2021}
Yaqing Chen, Zhenhua Lin, and Hans-Georg M{\"u}ller.
\newblock Wasserstein regression.
\newblock \emph{Journal of the American Statistical Association}, 0\penalty0
  (0):\penalty0 1--14, 2021.
\newblock \doi{10.1080/01621459.2021.1956937}.
\newblock URL \url{https://doi.org/10.1080/01621459.2021.1956937}.

\bibitem[Del~Moral and Niclas(2017)]{https://doi.org/10.48550/arxiv.1705.08561}
Pierre Del~Moral and Angele Niclas.
\newblock A taylor expansion of the square root matrix functional, 2017.
\newblock URL \url{https://arxiv.org/abs/1705.08561}.

\bibitem[Domazakis et~al.(2019)Domazakis, Drivaliaris, Koukoulas, Papayiannis,
  Tsekrekos, and Yannacopoulos]{Domazakis2019ClusteringMD}
G.~Domazakis, Dimosthenis Drivaliaris, Sotirios Koukoulas, G.~I. Papayiannis,
  Andrianos~E. Tsekrekos, and Athanasios~N. Yannacopoulos.
\newblock Clustering measure-valued data with wasserstein barycenters.
\newblock \emph{arXiv: Machine Learning}, 2019.

\bibitem[Dvinskikh and
  Tiapkin(2020)]{https://doi.org/10.48550/arxiv.2010.04677}
Darina Dvinskikh and Daniil Tiapkin.
\newblock Improved complexity bounds in wasserstein barycenter problem, 2020.
\newblock URL \url{https://arxiv.org/abs/2010.04677}.

\bibitem[Fei and Chen(2018)]{FeiChen2018}
Yingjie Fei and Yudong Chen.
\newblock Hidden integrality of sdp relaxation for sub-gaussian mixture models.
\newblock \emph{arXiv:1803.06510}, 2018.

\bibitem[Genevay et~al.(2018)Genevay, Peyre, and Cuturi]{pmlr-v84-genevay18a}
Aude Genevay, Gabriel Peyre, and Marco Cuturi.
\newblock Learning generative models with sinkhorn divergences.
\newblock In Amos Storkey and Fernando Perez-Cruz, editors, \emph{Proceedings
  of the Twenty-First International Conference on Artificial Intelligence and
  Statistics}, volume~84 of \emph{Proceedings of Machine Learning Research},
  pages 1608--1617. PMLR, 09--11 Apr 2018.
\newblock URL \url{https://proceedings.mlr.press/v84/genevay18a.html}.

\bibitem[Giraud and Verzelen(2018)]{GiraudVerzelen2018}
Christophe Giraud and Nicolas Verzelen.
\newblock Partial recovery bounds for clustering with the relaxed $k$means.
\newblock \emph{arXiv:1807.07547}, 2018.

\bibitem[Hütter and Rigollet(2019)]{HutterRigollet2019}
Jan-Christian Hütter and Philippe Rigollet.
\newblock Minimax estimation of smooth optimal transport maps, 2019.
\newblock URL \url{https://arxiv.org/abs/1905.05828}.

\bibitem[Janati et~al.(2020)Janati, Cuturi, and Gramfort]{pmlr-v119-janati20a}
Hicham Janati, Marco Cuturi, and Alexandre Gramfort.
\newblock Debiased {S}inkhorn barycenters.
\newblock In Hal~Daumé III and Aarti Singh, editors, \emph{Proceedings of the
  37th International Conference on Machine Learning}, volume 119 of
  \emph{Proceedings of Machine Learning Research}, pages 4692--4701. PMLR,
  13--18 Jul 2020.
\newblock URL \url{https://proceedings.mlr.press/v119/janati20a.html}.

\bibitem[Kim and Pass(2017)]{KIM2017640}
Young-Heon Kim and Brendan Pass.
\newblock Wasserstein barycenters over riemannian manifolds.
\newblock \emph{Advances in Mathematics}, 307:\penalty0 640--683, 2017.
\newblock ISSN 0001-8708.
\newblock \doi{https://doi.org/10.1016/j.aim.2016.11.026}.
\newblock URL
  \url{https://www.sciencedirect.com/science/article/pii/S0001870815304643}.

\bibitem[Le et~al.(2021)Le, Nguyen, Nguyen, Pham, Bui, and
  Ho]{NEURIPS2021_b80ba738}
Khang Le, Huy Nguyen, Quang~M Nguyen, Tung Pham, Hung Bui, and Nhat Ho.
\newblock On robust optimal transport: Computational complexity and barycenter
  computation.
\newblock In M.~Ranzato, A.~Beygelzimer, Y.~Dauphin, P.S. Liang, and J.~Wortman
  Vaughan, editors, \emph{Advances in Neural Information Processing Systems},
  volume~34, pages 21947--21959. Curran Associates, Inc., 2021.
\newblock URL
  \url{https://proceedings.neurips.cc/paper/2021/file/b80ba73857eed2a36dc7640e2310055a-Paper.pdf}.

\bibitem[Lloyd(1982)]{Lloyd1982_TIT}
Stuart Lloyd.
\newblock Least squares quantization in pcm.
\newblock \emph{IEEE Transactions on Information Theory}, 28:\penalty0
  129--137, 1982.

\bibitem[Lott(2008)]{Lott2008}
John Lott.
\newblock Some geometric calculations on wasserstein space.
\newblock \emph{Communications in Mathematical Physics}, 277\penalty0
  (2):\penalty0 423--437, 2008.
\newblock \doi{10.1007/s00220-007-0367-3}.
\newblock URL \url{https://doi.org/10.1007/s00220-007-0367-3}.

\bibitem[Lu and Zhou(2016)]{LuZhou2016}
Yu~Lu and Harrison Zhou.
\newblock Statistical and computational guarantees of lloyd's algorithm and its
  variants.
\newblock \emph{arXiv:1612.02099}, 2016.

\bibitem[MacQueen(1967)]{MacQueen1967_kmeans}
J.B. MacQueen.
\newblock Some methods for classification and analysis of multivariate
  observations.
\newblock \emph{Proc. Fifth Berkeley Sympos. Math. Statist. and Probability},
  pages 281--297, 1967.

\bibitem[McCann(1997)]{MCCANN1997153}
Robert~J. McCann.
\newblock A convexity principle for interacting gases.
\newblock \emph{Advances in Mathematics}, 128\penalty0 (1):\penalty0 153--179,
  1997.
\newblock ISSN 0001-8708.
\newblock \doi{https://doi.org/10.1006/aima.1997.1634}.
\newblock URL
  \url{https://www.sciencedirect.com/science/article/pii/S0001870897916340}.

\bibitem[Meila and Shi(2001)]{Meila01learningsegmentation}
Marina Meila and Jianbo Shi.
\newblock Learning segmentation by random walks.
\newblock In \emph{In Advances in Neural Information Processing Systems}, pages
  873--879. MIT Press, 2001.

\bibitem[Ng et~al.(2001)Ng, Jordan, and Weiss]{NgJordanWeiss2001_NIPS}
Andrew~Y. Ng, Michael~I. Jordan, and Yair Weiss.
\newblock On spectral clustering: Analysis and an algorithm.
\newblock In \emph{Advances in Neural Information Processing Systems}, pages
  849--856. MIT Press, 2001.

\bibitem[Peng and Wei(2007)]{PengWei2007_SIAMJOPTIM}
Jiming Peng and Yu~Wei.
\newblock Approximating $k$-means-type clustering via semidefinite programming.
\newblock \emph{SIAM J. OPTIM}, 18\penalty0 (1):\penalty0 186--205, 2007.

\bibitem[Rigollet and Weed(2019)]{RigolletWeed2019}
Philippe Rigollet and Jonathan Weed.
\newblock {Uncoupled isotonic regression via minimum Wasserstein
  deconvolution}.
\newblock \emph{Information and Inference: A Journal of the IMA}, 8\penalty0
  (4):\penalty0 691--717, 04 2019.
\newblock ISSN 2049-8772.
\newblock \doi{10.1093/imaiai/iaz006}.
\newblock URL \url{https://doi.org/10.1093/imaiai/iaz006}.

\bibitem[Santambrogio and Wang(2016)]{SANTAMBROGIO2016152}
Filippo Santambrogio and Xu-Jia Wang.
\newblock Convexity of the support of the displacement interpolation:
  Counterexamples.
\newblock \emph{Applied Mathematics Letters}, 58:\penalty0 152--158, 2016.
\newblock ISSN 0893-9659.
\newblock \doi{https://doi.org/10.1016/j.aml.2016.02.016}.
\newblock URL
  \url{https://www.sciencedirect.com/science/article/pii/S0893965916300726}.

\bibitem[Seguy and Cuturi(2015)]{SeguyCuturi2015_NIPS}
Vivien Seguy and Marco Cuturi.
\newblock Principal geodesic analysis for probability measures under the
  optimal transport metric.
\newblock In C.~Cortes, N.~Lawrence, D.~Lee, M.~Sugiyama, and R.~Garnett,
  editors, \emph{Advances in Neural Information Processing Systems}, volume~28.
  Curran Associates, Inc., 2015.
\newblock URL
  \url{https://proceedings.neurips.cc/paper/2015/file/f26dab9bf6a137c3b6782e562794c2f2-Paper.pdf}.

\bibitem[Solomon et~al.(2015)Solomon, de~Goes, Peyr\'{e}, Cuturi, Butscher,
  Nguyen, Du, and Guibas]{Solomon_2015}
Justin Solomon, Fernando de~Goes, Gabriel Peyr\'{e}, Marco Cuturi, Adrian
  Butscher, Andy Nguyen, Tao Du, and Leonidas Guibas.
\newblock Convolutional wasserstein distances: Efficient optimal transportation
  on geometric domains.
\newblock \emph{ACM Trans. Graph.}, 34\penalty0 (4), jul 2015.
\newblock ISSN 0730-0301.
\newblock \doi{10.1145/2766963}.
\newblock URL \url{https://doi.org/10.1145/2766963}.

\bibitem[van~der Vaart and Wellner(1996)]{vanderVaart1996}
Aad~W. van~der Vaart and Jon~A. Wellner.
\newblock \emph{Maximal Inequalities and Covering Numbers}, pages 95--106.
\newblock Springer New York, New York, NY, 1996.
\newblock ISBN 978-1-4757-2545-2.
\newblock \doi{10.1007/978-1-4757-2545-2_14}.
\newblock URL \url{https://doi.org/10.1007/978-1-4757-2545-2_14}.

\bibitem[Vempala and Wang(2004)]{Vempala04aspectral}
Santosh Vempala and Grant Wang.
\newblock A spectral algorithm for learning mixture models.
\newblock \emph{J. Comput. Syst. Sci}, 68:\penalty0 2004, 2004.

\bibitem[Verdinelli and Wasserman(2019)]{VerdinelliWasserman2019}
Isabella Verdinelli and Larry Wasserman.
\newblock {Hybrid Wasserstein distance and fast distribution clustering}.
\newblock \emph{Electronic Journal of Statistics}, 13\penalty0 (2):\penalty0
  5088 -- 5119, 2019.
\newblock \doi{10.1214/19-EJS1639}.
\newblock URL \url{https://doi.org/10.1214/19-EJS1639}.

\bibitem[Vershynin(2018)]{vershynin_2018}
Roman Vershynin.
\newblock \emph{High-Dimensional Probability: An Introduction with Applications
  in Data Science}.
\newblock Cambridge Series in Statistical and Probabilistic Mathematics.
  Cambridge University Press, 2018.
\newblock \doi{10.1017/9781108231596}.

\bibitem[Villani(2003)]{Villani2003_topics-in-ot}
C\'edric Villani.
\newblock \emph{Topics in optimal transportation}.
\newblock Graduate studies in mathematics. American mathematical society, 2003.

\bibitem[von Luxburg(2007)]{vanLuxburg2007_spectralclustering}
Ulrike von Luxburg.
\newblock A tutorial on spectral clustering.
\newblock \emph{Statistics and Computing}, 17\penalty0 (4):\penalty0 395--416,
  2007.

\bibitem[von Luxburg et~al.(2008)von Luxburg, Belkin, and
  Bousquet]{vonLuxburgBelkinBousquet2008_AoS}
Ulrike von Luxburg, Mikhail Belkin, and Olivier Bousquet.
\newblock Consistency of spectral clustering.
\newblock \emph{Annals of Statistics}, 36\penalty0 (2):\penalty0 555--586,
  2008.

\bibitem[Zhuang et~al.(2022)Zhuang, Chen, and Yang]{pmlr-v151-zhuang22a}
Yubo Zhuang, Xiaohui Chen, and Yun Yang.
\newblock Sketch-and-lift: scalable subsampled semidefinite program for k-means
  clustering.
\newblock In Gustau Camps-Valls, Francisco J.~R. Ruiz, and Isabel Valera,
  editors, \emph{Proceedings of The 25th International Conference on Artificial
  Intelligence and Statistics}, volume 151 of \emph{Proceedings of Machine
  Learning Research}, pages 9214--9246. PMLR, 28--30 Mar 2022.
\newblock URL \url{https://proceedings.mlr.press/v151/zhuang22a.html}.

\end{thebibliography}


\newpage
\appendix
\section{Additional details on application for real datasets in Section~\ref{subsec:realdata}}
\label{app:real_setup}
In this section, we provide more details of setups and results for real applications in
Section~\ref{subsec:realdata}. The results of error rates, $F_1$ scores and time costs are shown in Table~\ref{tab:Realdata}, Table~\ref{tab:Realdata_f1} and Table~\ref{tab:Realdata_tc} respectively, which are based on $10$ replicates.\footnote{We run all the simulations and experiments except for USPS datasets on the machine with Intel Core i7-10700K 3.80 GHz 64 bit 8-core 16 Tread Processor and 16 GB DDR4 Memory; run experiments on USPS datasets with 1.6 GHz Dual-Core Intel Core i5 and 8 GB 2133 MHz LPDDR3 Memory.} 

\begin{table}[h]
\caption{$F_1$ score (SD) for clustering three benchmark datasets: MNIST, Fashion-MNIST and USPS handwriting digits. MNIST$_1$ (MNIST$_2$) refers to the results of Case 1 (Case 2) for MNIST dataset.}
\vspace{-0.2cm}
\label{tab:Realdata_f1}
\begin{center}
\begin{tabular}{ccccc} 
\toprule
   & W-SDP & D-WKM  & B-WKM  & KM\\
\hline
MNIST$_1$  &  0.771 (0.044) & 0.842 (0.056) & 0.698 (0.067) & 0.708 (0.063)\\
MNIST$_2$  & 0.729 (0.049) & 0.814 (0.093) & 0.685 (0.031) & 0.647 (0.032)\\
Fashion-MNIST  & 0.919 (0.018) & 0.934 (0.036) & 0.817 (0.117) & 0.791(0.168)\\
USPS handwriting & 0.799 (0.019) &  0.835 (0.081)& 0.761 (0.060) & 0.689 (0.093)\\
 \bottomrule
\end{tabular}
\vspace{-0.2cm}
\end{center}
\end{table}

\begin{table}[h]
\caption{Time cost (SD) for clustering three benchmark datasets: MNIST, Fashion-MNIST and USPS handwriting digits. MNIST$_1$ (MNIST$_2$) refers to the results of Case 1 (Case 2) for MNIST dataset.}
\vspace{-0.2cm}
\label{tab:Realdata_tc}
\begin{center}
\begin{tabular}{ccccc} 
\toprule
   & W-SDP & D-WKM  & B-WKM  & KM\\
\hline
MNIST$_1$  & 525.13 (4.70) & 524.80 (4.92) & 388.87 (647.15)  & 0.01 (0.01)\\
MNIST$_2$  & 2187.66 (74.67) & 2160.91 (7.26) & 693.67 (142.57) & 0.02 (0.00)\\
Fashion-MNIST  & 849.24 (7.09) & 852.49 (8.60) & 463.28 (176.32) & 0.01 (0.00)\\
USPS handwriting & 1100.87 (19.13) & 1098.05 (16.41) & 317.12 (113.94) & 0.02 (0.01)\\
 \bottomrule
\end{tabular}
\vspace{-0.2cm}
\end{center}
\end{table}

First we run our Wasserstein SDP algorithm against Wasserstein $K$-means on the MNIST dataset for two cases. (1) For the first case, we choose two clusters: $G_1^\ast$ containing the number "0" and $G_2^\ast$ containing the number "5", so that the number of clusters is $K=2$ in the algorithms. The cluster sizes are unbalanced with $|G_1^\ast|/|G_2^\ast|=2,$ where we randomly choose $200$ number "0" and $100$ number of "5" for each repetition. (2) For the second case, we follow the same settings as case 1 except that randomly choose $400$ number "0" and $200$ number of "5" for each repetition. 

Next we considered benchmark dataset Fashion-MNIST 28$\times$28 containing $10$ clusters of 28$\times$28 greyscale images of clothes. Here we choose three clusters: $G_1$ containing the "T-shirt/top", $G_2$ containing the "Trouser" and $G_3$ containing the "Dress", so that the number of clusters is $K=3$ in the algorithms. The cluster sizes are unbalanced where we randomly choose $200, 100$ and $100$ number from $G_1,G_2$ and $G_3$ respectively for each repetition.

Finally, we consider the USPS handwriting dataset, analogous to MNIST, which contains digits automatically scanned from envelopes by the U.S. Postal Service containing a total of 9,298 16$\times$16 pixel grayscale samples. We choose three clusters: $G_1$ containing the number "0", $G_2$ containing the number "5" and $G_3$ containing the number "7", so that the number of clusters is $K=3$ in the algorithms. The cluster sizes are unbalanced where we randomly choose $200, 100$ and $100$ number from $G_1,G_2$ and $G_3$ respectively for each repetition.


Now if we look at the time cost for two cases on MNIST (MNIST$_1$ and MNIST$_2$) in Table~\ref{tab:Realdata_tc}, we can see that Wasserstein SDP (W-SDP), distance-based Wasserstein $K$-means (D-WKM) and barycenter-based Wasserstein $K$-means (B-WKM) all have time complexity issues when we enlarge $n$. The large variance for B-WKM for MNIST$_1$ is due to the convergence of the algorithm. The total iterations for B-WKM in case 1 achieves maximum iteration 100 for 1 replicate out of 10 total replicates. More arguments for time complexity can be found in Appendix~\ref{app:sub:sim_setup}.

\section{Additional details on simulation studies in Section~\ref{subsec:counterexample}}
\label{app:sub:sim_setup}

In this section, we provide more details of our simulation setups and results for Gaussian mixtures in Section~\ref{subsec:counterexample}.

We set $m_1=40\cdot r, m_2= r,m_3=20\cdot r$ which means that there are total $81\cdot r$ number of distributions in $G_1,$ $20\cdot r$ distributions in $G_2.$ The $r$ is set to be $1,2,3,$ where we have $n=101,202,303$ respectively. The mean for the Gaussian distributions are shown in the table below. The entries of covariance matrices for the Gaussian distributions are chosen to be $O(10^{-3})$ for $\mu_1,\mu_2$ and they are chosen to be $O(10^{-6})$ for $\mu_3$ and $\mu_4$. Then we scale down the distribution with scaling parameter equals $0.5.$  This ensures that with high probability, all the distributions will fall into the bounded range $[0,1]\times [0,1].$ 

The algorithm we use to get the barycenter is Frank-Wolfe algorithm with $200$ iterations. And we use Sinkhorn divergence to calculate the Wasserstein distance. The regularization parameters for both algorithms are chosen to be $10^{-3}$. To approximate the true distribution, first we divide $[0,1]\times [0,1]$ range into $80\times 80$ grids, then we randomly sample $600$ samples each time and count the number of times it falls into certain grid to approximate the distribution. The results show us that for each $n$ and each iterations among $50$ repetitions, all the distributions in $G_2^\ast$ will be assigned to same cluster, so it will be reasonable to define that $\mu_3$ is misclassified if any copy of them are in the same cluster of an arbitrarily chosen $\mu$ from $G_2^\ast.$

The arrangement of mean for Gaussian mixture models shown in Table~\ref{tab:initial_weights2} indicates that the distributions are set based on Example~\ref{ex:failure_barycenter_Wasserstein_Kmeans}. Recall that in Section~\ref{subsec:counterexample}, $\Delta_\ast:=\max_{k=1,2}\max_{i,j\in G_k} W_2(\mu_i, \mu_j)$ and $\Delta^\ast:=\min_{i\in G_1, j\in G_2} W^2(\mu_i, \mu_j)$ are the maximum within-cluster distance and the minimum between-cluster distance respectively. Table~\ref{tab:results_conterexample1} shows that $\Delta_\ast< \Delta^\ast$ on average and $\Delta_\ast< \Delta^\ast$ for around $80\%$ among $50$ repetitions. So we can expect Wasserstein SDP to correctly cluster all data points in the Wassertein space. From Table~\ref{tab:results_conterexample12} we can observe that in our settings the time cost for Wasserstein SDP and distance-based Wasserstein $K$-means is relatively lower than the time cost for barycenter-based Wasserstein $K$-means. But we can see that as $n$ increases, the time cost for B-WKM grows almost linearly w.r.t. $n$ while almost quadratically for W-SDP and D-WKM. Thus we should expect relatively higher time cost for W-SDP and D-WKM when $n$ is sufficiently large, where we can consider several methods to bring down the time cost (e.g., subsampling-based method for SDP from ~\cite{pmlr-v151-zhuang22a}).

Computationally speaking, the calculations of Wasserstein distances and barycenters are usually based on one-step discretization and one-step application of entropic regularization methods such as Sinkhorn (\cite{pmlr-v84-genevay18a}, ~\cite{pmlr-v119-janati20a}). \cite{https://doi.org/10.48550/arxiv.2010.04677} shows that the complexity of calculating barycenters should be of the order $O(nd^2/\epsilon^2))$ or $O(ng^4/\epsilon^2)),$ where $n$ is the total number of distributions, $d=g^2$ is the discretization size, e.g. $g=28$ for MNIST datasets and $\epsilon$ is the numerical accuracy; while~\cite{NEURIPS2021_b80ba738} gives a $O(d^2/\epsilon))$ or $O(g^4/\epsilon))$ complexity algorithm for calculating the Wasserstein distance on robust optimal transport.

\begin{table}[h]
\caption{Positions $(x,y)\in \mathbb{R}^2$ of means for two-dimensional mixture of Gaussian distributions for the counter example in Section~\ref{subsec:counterexample}. }
\label{tab:initial_weights2}
\vspace{-0.2cm}
\begin{center}
\begin{tabular}{ccccccccc} 
\toprule
  & $a_{1,1}$ & $a_{1,2}$ & $a_{2,1}$ & $a_{2,2}$ & $a_{3,1}$ & $a_{3,2}$ & $a_{4,1}$ & $a_{4,2}$ \\
\hline
$x$ & 0.75 & 0.25 & 0.75 & 0.25 & 0.9 & 0.9 & 1.3 & 1.3 \\
$y$ & 1.15 & 0.85 & 0.85  & 1.15  & 0.85  & 1.15 & 0.75 & 1.25 \\[-0.6ex]
 \bottomrule
\end{tabular}
\end{center}

\end{table}

\begin{table}[h]
\caption{The time cost with standard deviation shown in parentheses for the counter example. TC: Time cost, W-SDP: Wasserstein SDP, D-WKM: Distance-based Wasserstein $K$-means, B-WKM: Barycenter-based Wasserstein $K$-means.}
\label{tab:results_conterexample12}
\vspace{-0.2cm}
\begin{center}
\begin{tabular}{cccc} 
\toprule
$n$  & TC for W-SDP (SD) &TC for D-WKM&  TC for B-WKM (SD)\\
\hline
101 &  14.50\ (0.5873) & 14.15\ (0.5132)  & 181.1\ (372.4)  \\
202 &  56.94\ (1.490)  & 54.98\ (1.516)   &341.0\ (136.2) \\
303  & 128.4\ (3.640)  & 123.9\ (3.606) &549.2\ (200.2) \\[-0.6ex]
 \bottomrule
\end{tabular}
\end{center}
\end{table}

\begin{table}[h]
\caption{Estimated Wasserstein distances with standard deviation shown in parentheses and frequency of $\Delta^\ast >\Delta_\ast$ for the counter example.}
\label{tab:results_conterexample1}
\vspace{-0.2cm}
\begin{center}
\begin{tabular}{cccc} 
\toprule
$n$  &  $\Delta_\ast$ & $\Delta^\ast$ & Frequency of $\Delta_\ast<\Delta^\ast$ \\
\hline
101 & 0.1978\ (0.0055) & 0.2046\ (0.0050) & 0.8200\\
202 &  0.1990\ (0.0058)  & 0.2050\ (0.0051)& 0.8200\\
303 & 0.1996\ (0.0067)  & 0.2052\ (0.0050)& 0.7600\\[-0.6ex]
 \bottomrule
\end{tabular}
\end{center}
\end{table}

\section{Background on optimal transport}
\label{app:sec:background}

The optimal transport (OT) problem (a.k.a. the Monge problem) is to find an optimal map $T^* : \bR^p \to \bR^p$ for transporting a source distribution $\mu_0$ to a target distribution $\mu_1$ that minimizes some cost function $c : \bR^p \times \bR^p \to \bR$:
\begin{equation}
    \label{eqn:monge_problem}
    \min_{T : \bR^p \to \bR^p} \left\{ \int_{\bR^p} c(x, T(x)) d \mu_0(x) :  T_\sharp \mu_0 = \mu_1 \right\},
\end{equation}
where $T_\sharp \mu$ denotes the pushforward measure defined by $(T_\sharp \mu) (B) = \mu(T^{-1}(B))$ for measurable subset $B \subset \bR^p$. A standard example of the cost function is the quadratic cost $c(x,y) = \|x-y\|_2^2$. The Monge problem~\eqref{eqn:monge_problem} with the quadratic cost induces a metric, known as the \emph{2-Wasserstein distance}, on the space $\cP_2(\bR^p)$ of probability measures on $\bR^p$ with finite second moments. In particular, the 2-Wasserstein distance can be expressed in the relaxed Kantorovich form:
\begin{equation}
    \label{eqn:kantorovich_problem}
    W_2^2(\mu_0, \mu_1) := \min_{\gamma} \left\{ \int_{\bR^p \times \bR^p} \|x - y\|_2^2 d \gamma(x,y) \right\},
\end{equation}
where minimization over $\gamma$ runs over all possible couplings with marginals $\mu_0$ and $\mu_1$~\citep{Villani2003_topics-in-ot}. It is well-known from Brenier's theorem~\citep{Brenier1991} that if the source measure $\mu_0$ does not charge on small subsets of $\bR^p$ (i.e., subsets of Hausdorff dimension at most $p-1$), then there exists a unique $\mu_0$-almost everywhere OT map $T^*$ solving~\eqref{eqn:monge_problem}. That is, $T^*_\sharp \mu_0 = \mu_1$ and
\begin{equation*}
    W_2^2(\mu_0, \mu_1) = \int_{\bR^p} \|x - T^*(x) \|_2^2 d \mu_0(x).
\end{equation*}
Let $(\mu_t)_{t=0}^1$ be the constant-speed geodesic connecting $\mu_0, \mu_1 \in \cP_2(\bR^p)$. Then for any $\nu \in \cP_2(\bR^p)$ and $t \in [0, 1]$, we have
\begin{equation}
    \label{eqn:semiconcavity_ineq}
    W_2^2(\mu_t, \nu) \geq (1-t) W_2^2(\mu_0, \nu) + t W_2^2(\mu_1, \nu) - t (1-t) W_2^2(\mu_0, \mu_1).
\end{equation}
The above semiconcavity inequality~\eqref{eqn:semiconcavity_ineq} can be interpreted as that the Wasserstein space $\cP_2(\bR^p)$ is a \emph{positive curved} metric space (PC-space) in the sense of Alexandrov (cf. Section 7.3 and Section 12.3 in~\cite{ambrosio2005gradient}).

\section{Additional proofs}
In this section, we will give detailed proofs for Example~\ref{ex:wasserstein_barycenter_support}, Lemma~\ref{lem:counter_example} and Theorem~\ref{thm:separation_upper_bound_gaussians}. For the proof of Theorem~\ref{thm:separation_upper_bound_gaussians}, we will first introduce the main part and put the rest proofs of corresponding lemmas at the end of this section to make it clear.

\subsection{Proof of Example~\ref{ex:wasserstein_barycenter_support}}
\label{app:subsec:counterexample_irregularity}
Recall that $\mu_0(s) = (1-s) / 2$ and $\mu_1(s) = (1+s) / 2$ are probability densities supported on the line segments $L_0 = \{ (s, a s) : s \in [-1, 1] \}$ and $L_1 = \{ (s, -a s) : s \in [-1, 1] \}$ for some $a \in (0, 1)$, respectively. To derive the optimal transport (OT) map $T$ from $\mu_0$ to $\mu_1$, it suffices to consider the one-dimensional OT problem by parameterization of $T : [-1, 1] \to [-1, 1]$ identified via $(s, a s) \mapsto (T(s), -a T(s))$. Then our goal is to find the solution to the following optimization problem
\begin{equation*}
    \begin{gathered}
     \min_{T : T_\sharp \mu_0 = \mu_1} \int_{-1}^1 \|(s, -a s) - (T(s), -a T(s)) \|_2^2 d \mu_0(s) \\
     = \min_{T : T_\sharp \mu_0 = \mu_1} \int_{-1}^1 [ (s - T(s))^2 + a^2 (s + T(s))^2 ] d \mu_0(s) \\
     =  (1 - a^2) \times \min_{T : T_\sharp \mu_0 = \mu_1} \int_{-1}^1 \left[ \sqrt{1+a^2 \over 1-a^2} T(s) - s \right]^2 d \mu_0(s) + \text{constant},
    \end{gathered}
\end{equation*}
where the constant does not depend on $T$. Now rescale the distribution density $\mu_1$ to
\begin{equation*}
    \tilde\mu_1(s) = \sqrt{1-a^2 \over 1+a^2} \; \mu_1\left( \sqrt{1-a^2 \over 1+a^2} s \right) \quad \text{for } s \in \left[ -\sqrt{1+a^2 \over 1-a^2}, \; \sqrt{1+a^2 \over 1-a^2} \right],
\end{equation*}
and define the transport map $\tilde{T} = \sqrt{1+a^2 \over 1-a^2} \; T$ on $[-1, 1]$. To find the OT map $T$ such that $T_\sharp \mu_0 = \mu_1$, it suffices to find the OT map $\tilde{T}$ such that $\tilde{T}_\sharp \mu_0 = \tilde{\mu}_1$, i.e.,
\begin{equation*}
    \min_{\tilde{T} : \tilde{T}_\sharp \mu_0 = \tilde{\mu}_1} \int_{-1}^1 [ \tilde{T}(s) - s ]^2 d \mu_0(s),
\end{equation*}
whose solution is known as the \emph{quantile transform} for one-dimensional distributions. Specifically, let
\begin{equation*}
    F_0(s) = \int_{-1}^s \mu_0(t) d t = {1\over2} \left( s - {1\over2} s^2 + {3 \over 2} \right) \quad \text{for } s \in [-1, 1],
\end{equation*}
be the cumulative distribution function (cdf) of the density $\mu_0$ and
\begin{equation*}
    \tilde{F}_1(s) = \int_{-\sqrt{1+a^2 \over 1-a^2}}^s \tilde{\mu}_1(t) d t = {1\over2} \left( \sqrt{1-a^2 \over 1+a^2} s + {1\over2} \cdot {1-a^2 \over 1+a^2} s^2 + {1\over2} \right) \quad \text{for } s \in \left[ -\sqrt{1+a^2 \over 1-a^2}, \; \sqrt{1+a^2 \over 1-a^2} \right],
\end{equation*}
be the cdf of the density $\tilde{\mu}_1$. It is easy to find that
\begin{equation*}
    \tilde{F}_1^{-1}(y) = \sqrt{1+a^2 \over 1-a^2} (\sqrt{4y} - 1) \quad \text{for } y \in [0, 1].
\end{equation*}
Then the OT map $\tilde{T}$ from $\mu_0$ to $\tilde{\mu}_1$ is given by
\begin{equation*}
    \tilde{T}(s) = \tilde{F}_1^{-1} \circ F_0(s) = \sqrt{1+a^2 \over 1-a^2} [-1 + \sqrt{4 - (1-s)^2}], \quad s \in [-1, 1].
\end{equation*}
This gives the OT map $T$ from $\mu_0$ to $\mu_1$ (in the one-dimensional parameterization form) as
\begin{equation}
    T(s) = -1 + \sqrt{4 - (1-s)^2}.
\end{equation}
Thus, the OT map $T$ from $\mu_0$ to $\mu_1$ as (degenerate) probability distribution in $\bR^2$ is given by
\begin{equation*}
    T(s, a s) = \big( -1+\sqrt{4-(1-s)^2}, \ \ -a \cdot (-1+\sqrt{4-(1-s)^2}) \big).
\end{equation*}

\subsection{Proof of Lemma~\ref{lem:counter_example} in Section~\ref{subsec:centroid-based_Wasserstein_Kmeans}}
Recall the settings as following
\begin{align*}
    \mu_1&=0.5\,\delta_{(x,y)}+0.5\,\delta_{(-x,-y)}, \quad \mu_2=0.5\,\delta_{(x,-y)}+0.5\,\delta_{(-x,y)}, \\
    \mu_3&=0.5\,\delta_{(x+\epsilon_1,y)}+0.5\,\delta_{(x+\epsilon_1,-y)}, \quad\mbox{and}\quad  \mu_4=0.5\,\delta_{(x+\epsilon_1+\epsilon_2,y)}+0.5\,\delta_{(x+\epsilon_1+\epsilon_2,-y)},
\end{align*}
where $\delta_{(x,y)}$ denotes the point mass measure at point $(x,y)\in \mathbb{R}^2$, and $(x,y,\epsilon_1,\epsilon_2)$ are positive constants.

\emph{Lemma~\ref{lem:counter_example}} (\textbf{Configuration characterization}).
If $(x,y,\epsilon_1,\epsilon_2)$ satisfies
\[
y^2<\min\{x^2, 0.25\,\Delta_{\epsilon_1,x} \} \quad\mbox{and}\quad \Delta_{\epsilon_1,x}<\epsilon_2^2<\Delta_{\epsilon_1,x}+y^2,
\]
where $\Delta_{\epsilon_1,x}:=\epsilon_1^2+2x^2+2x\epsilon_1$, then for all sufficiently large $m$ (number of copies of $\mu_1$ and $\mu_2$),
\[
W_2(\mu_3, \mu^\ast_2)<W_2(\mu_3, \mu^\ast_1) \qquad\mbox{and} \quad \underbrace{\max_{k=1,2}\max_{i,j\in G_k} W_2(\mu_i, \mu_j)}_{\mbox{\small largest within-cluster distance}} \! < \! \underbrace{\min_{i\in G_1, j\in G_2} W_2(\mu_i, \mu_j),}_{\mbox{\small least between-cluster distance}}
\]
where $\mu^\ast_k$ denotes the Wasserstein barycenter of cluster $G_k$ for $k=1,2$.
 
\emph{Proof.} For any $w_i\in \mathbb{R}^2,i=1,2,3,4,$ let $\mu=0.5\, \delta_{w_1}+0.5\, \delta_{w_2},\ \nu=0.5\, \delta_{w_3}+0.5\, \delta_{w_4}$.  By definition of Wasserstein distance we can show that
\[
W_2^2(\mu, \nu)=0.5\min\{\|w_1-w_3\|^2+\|w_2-w_4\|^2,\|w_1-w_4\|^2+\|w_2-w_3\|^2\}.
\]
Let $\mu_0=0.5\,\delta_{(x,0)}+0.5\,\delta_{(-x,0)}$, by algebraic calculation it is direct to check
\[
W_2(\mu_3, \mu^\ast_2)<W_2(\mu_3, \mu_0) \qquad\mbox{and} \quad \underbrace{\max_{k=1,2}\max_{i,j\in G_k} W_2(\mu_i, \mu_j)}_{\mbox{\small largest within-cluster distance}} \! < \! \underbrace{\min_{i\in G_1, j\in G_2} W_2(\mu_i, \mu_j),}_{\mbox{\small least between-cluster distance}}
\]
once plugging in the assumptions. So we only need to show that $\forall \varepsilon, \exists M$, s.t. when $m>M$ we have $W^2_2(\mu_3, \mu^\ast_1)\ge W^2_2(\mu_3, \mu_0)-\varepsilon$. For notation simplicity, let $v_x=(x,0),v_{-x}=(-x,0),v_1=(x,y),v_2=(-x,-y),v_3=(x,-y),v_4=(x,-y).$ By definition we know there exist measures $\xi_i,i=1,2,3,4,$ s.t. 
\[
W_2^2(\mu^\ast_1,\mu_1)=\int \|v-v_1\|^2 d \xi_1(v) +\int \|v-v_2\|^2 d \xi_2(v), 
\]
\[
W_2^2(\mu^\ast_1,\mu_2)=\int \|v-v_3\|^2 d \xi_3(v) +\int \|v-v_4\|^2 d \xi_4(v), 
\]
where $\mu^\ast_1= \xi_1+\xi_2= \xi_3+\xi_4$ with $\xi_i(\mathbb{R}^2)=0.5,\forall i.$ Furthermore, if we define $\xi_{i,j}=\xi_i\cdot \xi_j/\mu^\ast_1,i\in\{1,2\},j\in\{3,4\},$ then $\xi_i=\xi_{i,3}+\xi_{i,4}, \xi_j=\xi_{1,j}+\xi_{2,j},i\in\{1,2\},j\in\{3,4\}.$ Thus
\begin{align*}
W_2^2(\mu^\ast_1,\mu_1)+W_2^2(\mu^\ast_1,\mu_2)&=\sum_{i=1}^4 \int \|v-v_i\|^2 d \xi_i(v)\\
&=\sum_{i\in\{1,2\},j\in\{3,4\}} \int \|v-v_i\|^2 +\|v-v_j\|^2 d \xi_{i,j}(v).
\end{align*}
Now suppose $t=\|v-v_x\|$, by algebraic calculation we can get 
\[
\|v-v_1\|^2+\|v-v_3\|^2=t^2+2y^2.
\]
Choose $T>0$ s.t. $ T^2< \min\{2x^2-2y^2,y^2\},$ then we have
\begin{align*}
&W_2^2(\mu^\ast_1,\mu_1)+W_2^2(\mu^\ast_1,\mu_2)=\sum_{i\in\{1,2\},j\in\{3,4\}} \int \|v-v_i\|^2 +\|v-v_j\|^2 d \xi_{i,j}(v)\\
&\le \int_{B_T(v_{x})} \|v-v_1\|^2 +\|v-v_3\|^2 d \xi_{1,3}(v) +\int_{B_T(v_{-x})} \|v-v_2\|^2 +\|v-v_4\|^2 d \xi_{2,4}(v)\\
&+(T^2+2y^2)(1-\xi_{1,3}(B_T(v_{x}))-\xi_{2,4}(B_T(v_{-x})))\\
&= \int_{B_T(v_{x})} t_1(v)^2 d \xi_{1,3}(v) +\int_{B_T(v_{-x})} t_2(v)^2 d \xi_{2,4}(v)+2y^2+T^2(1-\xi_{1,3}(B_T(v_{x}))-\xi_{2,4}(B_T(v_{-x}))),
\end{align*}
where $B_t(v)$ stands for the ball with radius $t$ centered at $v$, $t_1(v):=\|v-v_x\|,t_2(v):=\|v-v_{-x}\|.$ On the other hand, by definition we know that
\begin{align*}
&m\cdot W_2^2(\mu^\ast_1,\mu_1)+m\cdot W_2^2(\mu^\ast_1,\mu_2) +W_2^2(\mu^\ast_1,\mu_3)\\
&\le m\cdot W_2^2(\mu_0,\mu_1)+m\cdot W_2^2(\mu_0,\mu_2) +W_2^2(\mu_0,\mu_3)\\
&=m\cdot (2y^2)+C,
\end{align*}
where $C:=W_2^2(\mu_0,\mu_3).$ So we have $W_2^2(\mu^\ast_1,\mu_1)+W_2^2(\mu^\ast_1,\mu_2)\le 2y^2+C/m.$ i.e.,
\[
\int_{B_T(v_{x})} t_1(v)^2 d \xi_{1,3}(v) +\int_{B_T(v_{-x})} t_2(v)^2 d \xi_{2,4}(v)+T^2(1-\xi_{1,3}(B_T(v_{x}))-\xi_{2,4}(B_T(v_{-x})))\le \frac{C}{m}.
\]
So we have
\[ 
\int_{B_T(v_{x})} t_1(v)^2 d \xi_{1,3}(v)\le \frac{C}{m},\ \  \int_{B_T(v_{-x})} t_2(v)^2 d \xi_{2,4}(v)\le \frac{C}{m}, 
\]
\[
0.5-\xi_{1,3}(B_T(v_{x})) \le \frac{C}{T^2m},\ \ 0.5-\xi_{2,4}(B_T(v_{-x})) \le \frac{C}{T^2m}.
\]
Now suppose $v_{\epsilon_1}:=(x+\epsilon_1,y), v_{-\epsilon_1}:=(x+\epsilon_1,-y),$ note that $T^2< y^2<\epsilon_1^2+y^2$ and $W_2^2(\mu_3, \mu_0 )=0.5 \|v_{x}-v_{\epsilon_1}\|^2+0.5\|v_{-x}-v_{-\epsilon_1}\|^2$. By definition of Wasserstein distance and symmetry we have
\begin{align*}
W^2_2(\mu_3, \mu^\ast_1)&\ge \int_{B_T(v_{x})} (\|v_{x}-v_{\epsilon_1}\|-t_1(v))^2 d \xi_{1,3}(v)+\int_{B_T(v_{-x})} (\|v_{-x}-v_{\epsilon_1}\|-t_2(v))^2 d \xi_{2,4}(v)    \\
&\ge \|v_{x}-v_{\epsilon_1}\|^2 \xi_{1,3}(B_T(v_{x}))+\|v_{-x}-v_{\epsilon_1}\|^2 \xi_{2,4}(B_T(v_{-x}))\\
&-2\|v_{x}-v_{\epsilon_1}\|\int_{B_T(v_{x})} t_1(v) d \xi_{1,3}(v)-2\|v_{x}-v_{-\epsilon_1}\|\int_{B_T(v_{-x})} t_2(v) d \xi_{2,4}(v) \\
&\ge W_2^2(\mu_3, \mu_0 )-C^2_1\cdot \frac{C}{T^2m}-C^2_2\cdot \frac{C}{T^2m}\\
&-2C_1\int_{B_T(v_{x})} t_1(v) d \xi_{2,4}(v)-2C_2\int_{B_T(v_{-x})} t_2(v) d \xi_{2,4}(v),
\end{align*}
where $C_1=\|v_{x}-v_{\epsilon_1}\|,C_2=\|v_{-x}-v_{\epsilon_1}\|.$ Set $\forall \varepsilon>0$. Finally, by H\"older's inequality we have
\begin{align*}
W^2_2(\mu_3, \mu^\ast_1)&\ge W_2^2(\mu_3, \mu_0 )-C^2_1\cdot \frac{C}{T^2m}-C^2_2\cdot \frac{C}{T^2m}\\
&-2C_1\sqrt{\int_{B_T(v_{x})} t^2_1(v) d \xi_{2,4}(v)}-2C_2\sqrt{\int_{B_T(v_{-x})} t^2_2(v) d \xi_{2,4}(v)}\\
&\ge W_2^2(\mu_3, \mu_0 )-C^2_1\cdot \frac{C}{T^2m}-C^2_2\cdot \frac{C}{T^2m}-2C_1\sqrt{\frac{C}{m}}-2C_2\sqrt{\frac{C}{m}}\\
&\ge W_2^2(\mu_3, \mu_0 )-\varepsilon,
\end{align*}
for large $m$, as desired. \qed

\subsection{Proof of Theorem~\ref{thm:separation_upper_bound_gaussians} in Section~\ref{sec:SDP_relaxation}}
\emph{Theorem~\ref{thm:separation_upper_bound_gaussians}} (\textbf{Exact recovery for clustering Gaussians}). Let  $\Delta^2:=\min_{k\neq l} d^2(V^{(k)},V^{(l)})$ denote the minimal pairwise separation among clusters, $\bar{n}:=\max_{k\in[K]}n_k$ (and $\underline{n}:=\min_{k\in[K]}n_k$) the maximum (minimum) cluster size, and $m:=\min_{k\ne l}\frac{2n_kn_l}{n_k+n_l}$ the minimal pairwise harmonic mean of cluster sizes. Suppose the covariance matrix $V_i$ of Gaussian distribution $\nu_i=N(0, V_i)$ is independently drawn from model~\eqref{eqn:stat_model} for $i=1,2,\ldots,n$. Let $\beta\in(0,1)$. If the separation $\Delta^2$ satisfies
\[
    \Delta^2 >  \bar{\Delta}^2:\,=\frac{C_1 t^2} {\min\{(1-\beta)^2,\beta^2\}}\, \mathcal{V}\, p^2\log n,
\]
then the SDP~\eqref{eqn:wasserstein_kmeans_SDP} achieves exact recovery with probability at least $1-C_2 n^{-1}$, provided that
$$ \underline{n} \ge C_3 \log^2 n , \ \ t\le C_4\sqrt{\log n}/\big[(p+\log \bar{n})  \mathcal{V}^{1/2}T_v^{1/2}\big], \ \ n/m\le C_5\log n,$$
where  $\mathcal{V}=\max_k \left\|V^{(k)} \right\|_{\mbox{\scriptsize \rm op}}$, $T_v=\max_k \mbox{\rm Tr}\big[\big(V^{(k)}\big)^{-1}\big]$, and $C_i,i=1,2,3,4,5$ are constants.

\begin{lem}[\textbf{Dual argument for SDP} (Section B in \cite{chen2021cutoff})] 
\label{lem:Dual_argument_for_SDP}
The sufficient condition for $Z^*=\sum_{k\in[K]}\frac{1}{n_k}1_{G_k}1_{G_k}^T$ to be the unique solution of the SDP problem is to find $(\lambda,\alpha,B)$ s.t.
\begin{align*}
&(C_1) \ \ B\ge 0 \ (B_{G_kG_k}=0, B_{G_kG_l}>0, \forall k\ne l), \\
&(C_2) \ \ W_n:=\lambda Id+\frac{1}{2}(1\alpha^T+\alpha 1^T)-A-B\succeq 0, \\
&(C_3) \ \ \mbox{\rm Tr}(W_nZ^*)=0,\\
&(C_4) \ \ \mbox{\rm Tr}(BZ^*)=0,\\
\end{align*}
which implies that 
\[
\alpha_{G_k}=\frac{2}{n_k}A_{G_kG_k}1_{n_k}-\frac{\lambda}{n_k}1_{n_k}-\frac{1}{n_k^2}(1_{n_k}^TA_{G_kG_k} 1_{n_k}).
\]
\begin{align*}
[B_{G_lG_k}1_{n_k}]_j=-\frac{n_l+n_k}{2n_l}\lambda&+\frac{n_k}{2}\left[\frac{1}{n_l^2}\sum_{s,r\in G_l} d^2(V_s,V_r)- \frac{1}{n_k^2}\sum_{s,r\in G_k} d^2(V_s,V_r)\right]\\
&+n_k\left[\frac{1}{n_k}\sum_{r\in G_k} d^2(V_j,V_r)- \frac{1}{n_l}\sum_{r\in G_l} d^2(V_j,V_r)\right],
\end{align*}
for $k\ne l, \ j\in G_l$.

\end{lem}
\emph{Remarks.} It can be justified that if we can find $(\lambda, B)$ satisfying above equations, then $(C_3),(C_4)$ will hold automatically. Details can be found in Section B in \cite{chen2021cutoff}.

Now we will proof the main theorem by two steps. First we will provide a lower bound for $[B_{G_lG_k}1_{n_k}]_j$. Similar to the argument from~\cite{chen2021cutoff}, we want to set $\lambda$ properly such that $(C_1)$ can hold.  In the next step we will try to verify that the choice of $(\lambda,\alpha,B)$ and the conditions on the signals could actually imply $(C_2)$. And since number of clusters $K$ is treated as fixed for most practical settings, we will not emphasize $K=O(1)$.
\subsubsection{Proof of main result.}
\emph{Step 1} (\textbf{Construct} $(\lambda, B)$). 
Recall $[B_{G_lG_k}1_{n_k}]_j=-\frac{n_l+n_k}{2n_l}\lambda+n_k L,$
where $L$ equals
\[
\frac{1}{2}\left[\frac{1}{n_l^2}\sum_{s,r\in G_l} d^2(V_s,V_r)- \frac{1}{n_k^2}\sum_{s,r\in G_k} d^2(V_s,V_r)\right]+\left[\frac{1}{n_k}\sum_{r\in G_k} d^2(V_j,V_r)- \frac{1}{n_l}\sum_{r\in G_l} d^2(V_j,V_r)\right].
\]
For $L$ defined above, by Lemma~\ref{lem:Lower_bound_for_L}, we have 
\begin{align*}
L\ge d^2(V^{(l)},V^{(k)})-d(V^{(l)},V^{(k)})K_1-K_2,
\end{align*}
w.p. at least ($1-c/n^2$), where
\begin{align*}
&K_1= C\sqrt{\log n} t \mathcal{V}^{1/2}+Ct^2(p+\log \bar{n})\mathcal{V}T_v^{1/2} ,\\
&K_2= Ct^2 p^2\log n \mathcal{V},
\end{align*}
for some constant $C,c.$ Now we chose $\beta\in (0,1)$ and let $m:=\min_{k\ne l}\frac{2n_kn_l}{n_k+n_l}$. If we suppose
\[
\Delta\ge Ctp\sqrt{\log n} \mathcal{V}^{1/2}/(1-\beta),\ \ t\le C'\sqrt{\log n}/\big[(p+\log \bar{n})  \mathcal{V}^{1/2}T_v^{1/2}\big],
\]
for some constant $C,C'$, then we have
\[
(1-\beta) d^2(V^{(l)},V^{(k)})-d(V^{(l)},V^{(k)})K_1-K_2\ge0, \forall k\ne l,
\]
which implies that
\[
L\ge \beta d^2(V^{(l)},V^{(k)}).
\]
Define for $k\ne l$,
\[c_j^{(k,l)}:=[B_{G_lG_k}1_{n_k}]_j,\ j \in G_l,\]
\[r_i^{(k,l)}:=[1^T_{n_l}B_{G_lG_k}]_i,\ i \in G_k, \]
\[t^{(k,l)}:=1^T_{n_l}B_{G_lG_k}1_{n_k}, \]
\[(B^\#_{G_lG_k})_{ij}:=r_i^{(k,l)}c_j^{(k,l)}/t^{(k,l)}.\]
And define $(B^\#_{G_lG_l})_{ij}:=0, \forall l.$ By setting $\lambda=\frac{\beta}{4}m\Delta^2,$ further we have
\[
c_j^{(k,l)}\ge \frac{\beta}{2}n_k d^2(V^{(l)},V^{(k)}), r_i^{(k,l)}\ge \frac{\beta}{2}n_l d^2(V^{(l)},V^{(k)}), t^{(k,l)}\ge\frac{\beta}{2}n_ln_k d^2(V^{(l)},V^{(k)}),
\]
which implies that $(B^\#_{G_lG_k})_{ij}>0,\forall i \in G_k, j \in G_l.$ And $
[B_{G_lG_k}1_{n_k}]_j=[B^\#_{G_lG_k}1_{n_k}]_j,$
which means we can construct $B^\#$ based on $[B_{G_lG_k}1_{n_k}]_j$ with $[B_{G_lG_k}1_{n_k}]_j=[B^\#_{G_lG_k}1_{n_k}]_j$. So essentially, they are the same in the sense that we only care about they quantity through $[B_{G_lG_k}1_{n_k}]_j$. And thus for notation simplicity, we will use the symbol $B$ instead of $B^\#.$

\emph{Step 2} (\textbf{Verify the condition for $W_n$ in $(C_2)$}). Next we would like to find sufficient condition for ($C_2$), i.e., 
$$v^T W_n v\ge 0, \forall v\in \Gamma_K:= \text{span}\{1_{G_k}:k\in[K]\}^\bot, \|v\|=1.$$
 Note that $v^T W_n v=\lambda -  v^T A v-v^T B v\ge \lambda -v^T B v.$ And by definition as well as simple calculation we have
\[
v^T B v=\sum_{k=1}^K\sum_{l\ne k}\frac{1}{t^{(k,l)}}\left(\sum_{i\in G_k} v_i r_i^{(k,l)} \right)\left(\sum_{j\in G_l} v_j c_j^{(k,l)} \right),
\]
\[
\sum_{j\in G_l} v_j c_j^{(k,l)} =n_k \sum_{j\in G_l}\left(\frac{1}{n_k}\sum_{r\in G_k} d^2(V_j,V_r)- \frac{1}{n_l}\sum_{r\in G_l} d^2(V_j,V_r)  \right)v_j.
\]
Further note that 
\[
\frac{1}{n_k}\sum_{r\in G_k} d^2(V_j,V_r)- \frac{1}{n_l}\sum_{r\in G_l} d^2(V_j,V_r)=d^2(V^{(l)},V^{(k)})+E_j^{(k,l)},
\]
where
\begin{align*}
E_j^{(k,l)}&=\left[\frac{1}{n_k}\sum_{r\in G_k} d^2(V_j,V_r)- d^2(V_j,V^{(k)})\right]+\left[d^2(V_j,V^{(k)})-d^2(V^{(l)},V^{(k)})\right]\\
&-\frac{1}{n_l}\sum_{r\in G_l} d^2(V_j,V_r).
\end{align*}
Then by triangle inequality and throwing away the last term of $E_j^{(k,l)},$ we have
\[\sum_{j\in G_l} v_j c_j^{(k,l)}=n_k \sum_{j\in G_l}E_j^{(k,l)}v_j\le n_k \sum_{j\in G_l}(E_{1,j}^{(k,l)}+E_{2,j}^{(k,l)})|v_j|, \]
where 
\[
E_{1,j}^{(k,l)}=\frac{1}{n_k}\sum_{r\in G_k} d^2(V^{(k)},V_r)+\left[\frac{2}{n_k}\sum_{r\in G_k} d(V^{(k)},V_r) d(V_j,V^{(k)})\right],
\]
\[
E_{2,j}^{(k,l)}=d^2(V^{(l)},V_j)+2d(V^{(l)},V_j)d(V^{(l)},V^{(k)}).
\]
If we set $\tilde{E}_{h,j}^{(k,l)}=E_{h,j}^{(k,l)}/d(V^{(l)},V^{(k)}), h=1,2,$ then the inequality can be written as
\[\sum_{j\in G_l} v_j c_j^{(k,l)}\le n_k d(V^{(l)},V^{(k)})\sum_{j\in G_l}(\tilde{E}_{1,j}^{(k,l)}+\tilde{E}_{2,j}^{(k,l)})|v_j|. \] 
By Lemma~\ref{lem:e1_upper_bound} we know
\[
\sum_{j\in G_l}\tilde{E}_{1,j}^{(k,l)}|v_j|\le C \mathcal{V}^{1/2} p t  \sqrt{n_l} \left(\sum_{j\in G_l} v_j^2 \right)^{1/2},
\]
w.p. $\ge 1-cn^{-2}.$ And by Lemma~\ref{lem:e2_upper_bound} we have
\[
\sum_{j\in G_l}\tilde{E}_{2,j}^{(k,l)}|v_j|\le Ct  \mathcal{V}^{1/2}p(\sqrt{n_l}+\log^2(n))\left(\sum_{j\in G_l} v_j^2 \right)^{1/2},
\]
w.p. $\ge 1-cn^{-1},$ for some constants $C,c.$ Now if we assume $\min_k n_k\ge C \log^2 n $ and notice that $ t^{(k,l)}\ge\frac{\beta}{2}n_ln_k d^2(V^{(l)},V^{(k)}),$ then further we can get
\begin{align*}
v^T B v &\le \sum_{k,l}\frac{n_kn_l}{ t^{(k,l)}}\sqrt{n_l}\sqrt{n_k}\left(\sum_{j\in G_l} v_j^2 \right)^{1/2}\left(\sum_{i\in G_k} v_i^2  \right)^{1/2}Ct^2  \mathcal{V}p^2\\
&\le \frac{Ct^2}{\beta} \left(\sum_l\sum_{j\in G_l} v_j^2 \right)^{1/2}\left(\sum_ln_l \right)^{1/2}\left(\sum_k\sum_{i\in G_k} v_i^2 \right)^{1/2}\left(\sum_kn_k \right)^{1/2} \mathcal{V}p^2\\
&=\frac{Ct^2}{\beta}p^2 n \mathcal{V},
\end{align*}
where the second inequality comes from Cauchy-Schwarz inequality.
So by assuming
\[
\Delta^2\ge \frac{C t^2}{\beta^2}   \mathcal{V} \cdot p^2 n/m ,
\]
for some constant $C,$
we have
 $$v^T W_n v\ge \lambda -v^T B v \ge\frac{\beta}{4}m\Delta^2-\frac{Ct^2}{\beta}p^2 n \mathcal{V}> 0.$$
Or it is sufficient to assume
\[
\Delta^2\ge \frac{C t^2}{\beta^2}   \mathcal{V} \cdot p^2 \log n ,
\]
if $n/m=O(\log n).$ To sum up, if we assume
\[
\Delta^2\ge \frac{C t^2}{\min\{1-\beta,\beta\}^2}  \mathcal{V} \cdot p^2 \log n ,
\]
then w.p. $\ge 1-c/n,$ we have $(C_1)-(C_4)$ hold by the construction of $(\lambda,B)$ for some constants $C,c.$ Finally by Lemma~\ref{lem:Dual_argument_for_SDP} we know the solution of SDP $Z^\ast$ exists uniquely, which is
\[
Z^*=\sum_{k\in[K]}\frac{1}{n_k}1_{G_k}1_{G_k}^T
\]
as desired.\qed

\emph{Remarks.} In our theorem, we focus on the relation between minimum cluster distance $\bar{\Delta}$ with number of distributions $n,$ which should be tight enough in the sense that $\bar{\Delta}\asymp{\sqrt{\log n}}$. This is the same order for the cut-off of exact recovery of SDP for Euclidean case from~\cite{chen2021cutoff}.

On the other hand, one sufficient condition for $V_i,i=1,\dots,n$ to be psd is $1-t\max_i \|X_i\|_{op}>0,$ which will hold w.p. $\ge 1-c/n^2$ if $t\le C/[\sqrt{p}+\sqrt{\log n}]$ for some constant $C,c.$ Recall from our assumption, $$t\le c\sqrt{\log n}/[(p+\log \bar{n})  \mathcal{V}^{1/2}T_v^{1/2}]\le c\sqrt{\log n}/(p+\log \bar{n}),$$
for some constant, since $T_v=\max_k \mbox{\rm Tr}((V^{(k)})^{-1})\ge p/\min_k \|V^{(k)}\|_{op}. $ This indicates that our bound for $t$ guarantees $V_i$ to be psd w.p. $\ge 1-c/n^2$ as $n\asymp\bar n.$ One may apply triangle inequality directly to Lemma~\ref{lem:Lower_bound_for_L} to get the upper bound of $t$ with less order in $p$, which is of less concern in our theorem, where we put more emphasis on the order in $n.$

\subsubsection{Proofs of lemmas.}
Before proving Lemma~\ref{lem:Lower_bound_for_L}, let us first look at the Taylor expansion for psd matrix.

\begin{lem}[\textbf{Taylor expansion for psd matrix} (Theorem 1.1 in~\cite{https://doi.org/10.48550/arxiv.1705.08561})]
\label{lem:Talyor_expansion_for_psd}
 The square root function $\varphi: \ Q\in \mathcal{S}_r^+ \mapsto Q^{1/2}$ is Fr\'echet differentiable at any order on  $\mathcal{S}_r^+$ with the first order derivative given for any $(A,H)\in \mathcal{S}_r^+ \times \mathcal{S}_r$ by the formula
 \[
 \nabla \varphi(A)\cdot H=\int_0^\infty e^{-t \varphi(A)} H e^{-t \varphi(A)} dt,
 \]
where $\mathcal{S}_r^+,  \mathcal{S}_r$ are the positive semi-definite matrix and symmetric matrix respectively. The higher order derivatives are defined inductively for any $n\ge 2$ by
\[
 \nabla^n \varphi(A)\cdot H=-\nabla \varphi(A) \cdot \left[  \sum_{p+q=n-2 \& p,q\ge0} \frac{n!}{(p+1)!(q+1)!} [\nabla^{p+1} \varphi(A)\cdot H ][\nabla^{q+1} \varphi(A)\cdot H] \right].
\]
Again from the same paper, we have the Taylor expansion for $\varphi(A)$:
\[
\varphi (A+H)=\varphi (A)+\sum_{1\le k \le n}\frac{1}{k!} \nabla^k \varphi(A)\cdot H +\bar{\nabla}^{n+1} \varphi[A,H],
\]
with 
\[
\bar{\nabla}^{n+1} \varphi[A,H]:=\frac{1}{n!} \int_0^1 (1-\epsilon)^n \nabla^{n+1} \varphi(A+\epsilon H)\cdot H d \epsilon.
\]
\end{lem}

\begin{cor}[\bf Decomposition of Wasserstein distance for Gaussians]
\label{cor:Decomposition_of_ws}

If we choose $n=1$ in Lemma~\ref{lem:Talyor_expansion_for_psd}, we have for $k\ne l, j\in G_l^\ast,$ and under the assumptions in the Theorem, the following expansion holds.
\begin{align*}
&d^2(V_j,V^{(k)})-d^2(V_j,V^{(l)})\\
=&d^2(V^{(l)},V^{(k)})+\left<\mathcal{A}(V^{(l)},V^{(k)}), t(X_jV^{(l)}+V^{(l)}X_j)+t^2X_jV^{(l)}X_j\right>\\
-&d^2(V_j,V^{(l)})-\Delta_0,
\end{align*}
\begin{align*}
&\frac{1}{n_k}\sum_{r\in G_k} d^2(V_j,V_r)-d^2(V_j,V^{(k)})\\
=&\left<\mathcal{A}(V_j,V^{(k)}),\frac{1}{n_k}\sum_{r\in G_k} t(X_rV^{(k)}+V^{(k)}X_r)+t^2X_rV^{(k)}X_r\right>-\Delta_1,
\end{align*}
where $\mathcal{A}(U,V):= Id-U^{1/2}(U^{1/2}V U^{1/2} )^{-1/2}U^{1/2},$ for $U,V:$ psd. And $\Delta_0\le0, \Delta_1\le 0,$ which are extra terms (high order terms in Lemma~\ref{lem:Talyor_expansion_for_psd}).
\end{cor}
\begin{proof}
By definition we know $d^2(V,U)=W_2^2(\nu,\mu),$ where $\nu\sim N(0, V), \mu\sim N(0,U).$ Thus
\[
d^2(V,U)=\mbox{\rm Tr}(V)+\mbox{\rm Tr}(U)-2\mbox{\rm Tr}[\sqrt{V^{1/2}UV^{1/2}}].
\]
So we have
\begin{align*}
&d^2(V_j,V^{(k)})-d^2(V^{(k)},V^{(l)})\\
&=\mbox{\rm Tr}[V_j-V^{(l)}]-2\mbox{\rm Tr}\left[\sqrt{(V^{(k)})^{1/2}V_j(V^{(k)})^{1/2}}-\sqrt{(V^{(k)})^{1/2}V^{(l)}(V^{(k)})^{1/2}}\right].
\end{align*}
On the other hand, by definition we know $V_j=(I+tX_j)V^{(l)}(I+tX_j)=V^{(l)}+t(X_jV^{(l)}+V^{(l)}X_j)+t^2X_jV^{(l)}X_j.$ Then by Lemma~\ref{lem:Talyor_expansion_for_psd} and note the second order remainder term is always negative semi-definite, we can directly get the results by first order Taylor expansion.
\end{proof}
\begin{lem}[\bf Norm for operator $\mathcal{A}$]
\label{lem:norm_for_operator_A}

We conclude that for any $U,V$: psd, we have
$$\|\mathcal{A}(U,V)\cdot V^{1/2} \|^2_F=\| V^{1/2}\cdot\mathcal{A}(U,V) \|^2_F=d^2(U,V).$$
\end{lem}
~\\
\emph{Proof.} Suppose we have the SVD
\[
U^{1/2}V^{1/2}=Q^T_1 \Sigma Q_2,
\]
then we have
\begin{align*}
\mathcal{A}(U,V)\cdot V^{1/2}&=(I-U^{1/2}(U^{1/2}VU^{1/2})^{-1/2})V^{1/2}\\
&=V^{1/2}-U^{1/2}Q_1^TQ_2,
\end{align*}
which implies that
\begin{align*}
\|\mathcal{A}(U,V)\cdot V^{1/2}\|_F^2&=\mbox{\rm Tr}(V)+\mbox{\rm Tr}(U)-2\mbox{\rm Tr}(V^{1/2}U^{1/2}Q_1^TQ_2 )\\
&=\mbox{\rm Tr}(V)+\mbox{\rm Tr}(U)-2\mbox{\rm Tr}(Q_2^T \Sigma Q_2 )\\
&=\mbox{\rm Tr}(V)+\mbox{\rm Tr}(U)-2\mbox{\rm Tr}(\sqrt{U^{1/2}VU^{1/2}} ).
\end{align*}
\qed

\begin{lem}[\bf Lower bound for $L$]
\label{lem:Lower_bound_for_L}
Recall that $L$ equals
\[
\frac{1}{2}\left[\frac{1}{n_l^2}\sum_{s,r\in G_l} d^2(V_s,V_r)- \frac{1}{n_k^2}\sum_{s,r\in G_k} d^2(V_s,V_r)\right]+\left[\frac{1}{n_k}\sum_{r\in G_k} d^2(V_j,V_r)- \frac{1}{n_l}\sum_{r\in G_l} d^2(V_j,V_r)\right],
\]
we have 
\begin{align*}
L\ge d^2(V^{(l)},V^{(k)})-d(V^{(l)},V^{(k)})K_1-K_2,
\end{align*}
w.p. at least ($1-c/n^2$), where
\begin{align*}
&K_1= C\sqrt{\log n} t \mathcal{V}^{1/2}+Ct^2(p+\log \bar{n})\mathcal{V}Tv^{1/2} ,\\
&K_2= Ct^2 p^2\log n \mathcal{V},
\end{align*}
for some constant $C,c$.
\end{lem}

\emph{Proof.} First note that we can decompose the term into three terms:
\begin{align*}
\frac{1}{n_k}\sum_{r\in G_k} d^2(V_j,V_r)- \frac{1}{n_l}\sum_{r\in G_l} d^2(V_j,V_r)=U_1-U_2+U_3,
\end{align*}
where
\begin{align*}
&U_1:=\frac{1}{n_k}\sum_{r\in G_k} d^2(V_j,V_r)- d^2(V_j,V^{(k)}),\\
&U_2:=\frac{1}{n_l}\sum_{r\in G_l} d^2(V_j,V_r)- d^2(V_j,V^{(l)})\\
&U_3:=d^2(V_j,V^{(k)})- d^2(V_j,V^{(l)}).
\end{align*}

If we further define $U_0:=\frac{1}{2}\left[\frac{1}{n_l^2}\sum_{s,r\in G_l} d^2(V_s,V_r)- \frac{1}{n_k^2}\sum_{s,r\in G_k} d^2(V_s,V_r)\right],$ then we have
\[
L=U_0+U_1-U_2+U_3.
\]
From Corollary~\ref{cor:Decomposition_of_ws} we know $U_1$ and $U_2$ can be lower bounded by throwing out the remainders $\Delta_1,\Delta_2$, i.e.,

\begin{align*}
U_1=&\frac{1}{n_k}\sum_{r\in G_k} d^2(V_j,V_r)-d^2(V_j,V^{(k)})\\
\ge&\left<\mathcal{A}(V_j,V^{(k)}),\frac{1}{n_k}\sum_{r\in G_k} t(X_rV^{(k)}+V^{(k)}X_r)+t^2X_rV^{(k)}X_r\right>,
\end{align*}

\begin{align*}
U_3=&d^2(V_j,V^{(k)})-d^2(V_j,V^{(l)})\\
\ge&d^2(V^{(l)},V^{(k)})+\left<\mathcal{A}(V^{(l)},V^{(k)}), t(X_jV^{(l)}+V^{(l)}X_j)+t^2X_jV^{(l)}X_j\right>\\
-&d^2(V_j,V^{(l)}).
\end{align*}
As for the $U_0$ and $U_3$, we choose to use triangle inequality to get a rough bound, i.e., by noting $d(V_j,V_r)\le d(V_j,V^{(l)}) +d(V_r,V^{(l)}),$ we have
\begin{align*}
U_2=& \frac{1}{n_l}\sum_{r\in G_l} d^2(V_j,V_r)-d^2(V_j,V^{(l)})\\
\le& \frac{1}{n_l}\sum_{r\in G_l} d^2(V^{(l)},V_r)+\frac{2}{n_l} d(V^{(l)},V_r) \sum_{r\in G_l} d(V_j,V^{(l)}).
\end{align*}
And
\begin{align*}
U_0=& \frac{1}{2}\left[\frac{1}{n_l^2}\sum_{s,r\in G_l} d^2(V_s,V_r)- \frac{1}{n_k^2}\sum_{s,r\in G_k} d^2(V_s,V_r)\right]\\
\ge& -\frac{1}{2}\frac{1}{n_l^2}\sum_{s,r\in G_k} (d(V_s,V^{(k)})+d(V_r,V^{(k)}))^2\\
&\ge -\frac{2}{n_l}\sum_{r\in G_k} d^2(V^{(k)},V_r).
\end{align*}
For the RHS of the inequality for $U_1$, it can be divided into two parts.
\[
Z_1^1:=\left<\mathcal{A}(V_j,V^{(k)}),\frac{1}{n_k}\sum_{r\in G_k} t(X_rV^{(k)}+V^{(k)}X_r)\right>
\]
and
\[
Z_2^1:=\left<\mathcal{A}(V_j,V^{(k)}),t^2\frac{1}{n_k}\sum_{r\in G_k}X_rV^{(k)}X_r\right>.
\]
 the first part is a Gaussian distribution whose variance can be bounded by $c_1 t^2\|\mathcal{A}(V_j,V^{(k)})V^{(k)}\|_F^2/n_k$, for some constant $c_1.$ By Gaussian tail bound $P(|N(0,1)|>u)\le e^{-u^2/2}, \forall u>0 $ and Lemma~\ref{lem:norm_for_operator_A}, we have
\begin{align*}
 |Z_1^1|&\le c_2 t \sqrt{\log n}  \|V^{(k)} \|^{1/2}/\sqrt{n_k}\cdot d(V_j,V^{(k)})\\
 &\le c_2 t\sqrt{\log n}   \mathcal{V}^{1/2}\cdot d(V_j,V^{(k)}),
\end{align*}
w.p. $\ge 1-c_3/n^2$, for some constant $c_2,c_3$. On the other hand, 
\begin{align*}
 |Z_2^1|&= t^2\left|\left<\mathcal{A}(V_j,V^{(k)})(V^{(k)})^{1/2},\frac{1}{n_k}\sum_{r\in G_k}X_rV^{(k)}X_r(V^{(k)})^{-1/2}\right>\right|\\
 &\le t^2\left\|\frac{1}{n_k}\sum_{r\in G_k}X_rV^{(k)}X_r(V^{(k)})^{-1/2}\right\|_F   \cdot d(V_j,V^{(k)})\\
 &\le t^2 \frac{1}{n_k}\sum_{r\in G_k}\left\| X_r \right\|^2 \|V^{(k)}\| \left\|(V^{(k)})^{-1/2}\right\|_F \cdot d(V_j,V^{(k)})\\
 &\le t^2  \max_{r\in G_k } \left\| X_r \right\|^2 \mathcal{V} Tv^{1/2}\cdot d(V_j,V^{(k)})\\
 &\le c_4t^2  (p+\log n) \mathcal{V} Tv^{1/2}\cdot d(V_j,V^{(k)}),
\end{align*}
w.p. $\ge 1-c_5/n^2$, for some constant $c_4,c_5$. The last inequality can be implied from union bound and Corollary 4.4.8 in~\cite{vershynin_2018}:
\[
\|X_r\|\le C(\sqrt{p}+u),\ \ \text{w.p.}\ge 1-4e^{-u^2}.
\]
Now by combining $Z_1^1,Z_2^1$ we have
\begin{align*}
U_1&\ge Z_1^1+Z_2^1\\
&\ge-\left[c_2 t\sqrt{\log n}   \mathcal{V}^{1/2}  +  c_4t^2  (p+\log n) \mathcal{V} Tv^{1/2}\right]\cdot d(V_j,V^{(k)}),
\end{align*}
w.p.  $\ge 1-(c_3+c_5)/n^2.$ 

~\\
For $U_0$, we have
\begin{align*}
U_0&\ge -\frac{2}{n_k}\sum_{r\in G_k} d^2(V^{(k)},V_r)\\
&=-\frac{2t^2}{n_k}\sum_{r\in G_k} \mbox{\rm Tr}(X_rV^{(k)}X_r)\\
&\ge -2t^2 \mathcal{V}\frac{1}{n_k}\sum_{r\in G_k}\mbox{\rm Tr}(X_r^2)\\
&\ge -c_6t^2 \mathcal{V}p^2,\\
\end{align*}
w.p. $\ge 1-c_7/n^2$ for some constant $c_6,c_7.$ The equation is a direct result by definition of Wasserstein distance for Gaussians:
\[ 
d^2(V^{(k)},V_r)=\mbox{\rm Tr}(V^{(k)})+\mbox{\rm Tr}(V_r)-2\mbox{\rm Tr}(\sqrt{(V^{(k)})^{1/2}V_r(V^{(k)})^{1/2}}).
\]
Note here
\begin{align*}
\sqrt{(V^{(k)})^{1/2}V_r(V^{(k)})^{1/2}}&=\sqrt{(V^{(k)})^{1/2}(I+tX_r)V^{(k)}(I+tX_r)(V^{(k)})^{1/2}}\\
&=(V^{(k)})^{1/2}(I+tX_r)(V^{(k)})^{1/2}.
\end{align*}
The last inequality can be derived through Bernstein’s inequality (Theorem 2.8.2 in~\cite{vershynin_2018}) by noting that $\mbox{\rm Tr}(X_r^2)$ is sub-exponential with mean $\mathbb{E}(\mbox{\rm Tr}(X_r^2))=p^2.$
Similar to the argument for $U_0, U_1$, after we apply high-dimensional bound for sub-Gaussian or sub-exponential distributions we can get bound for $U_2,U_3$:
\begin{align*}
U_2&\le \frac{1}{n_l}\sum_{r\in G_l} d^2(V^{(l)},V_r)+\frac{2}{n_l}  \sum_{r\in G_l} d(V_r,V^{(l)})d(V^{(l)},V_j)\\
&\le c_8 t^2 \mathcal{V} p^2\log n,
\end{align*}
w.p.  $\ge 1-c_9/n^2$, for some constant $c_8,c_9$. 
\begin{align*}
U_3&\ge d^2(V^{(l)},V^{(k)})+\left<\mathcal{A}(V^{(l)},V^{(k)}), t(X_jV^{(l)}+V^{(l)}X_j)+t^2X_jV^{(l)}X_j\right>\\
&-d^2(V_j,V^{(l)})\\
&\ge d^2(V^{(l)},V^{(k)})-\left[c_2 t\sqrt{\log n}   \mathcal{V}^{1/2}  +  c_4t^2  (p+\log \bar{n}) \mathcal{V} Tv^{1/2}\right]\cdot d(V^{(l)},V^{(k)})\\
&-c_{10}t^2(p+\log n)p\mathcal{V},
\end{align*}
w.p.  $\ge 1-c_{11}/n^2.$, for some constant $c_{10},c_{11}$. Lastly, by noting $d(V_j,V^{(k)})\le d(V^{(l)},V^{(k)})+d(V_j,V^{(l)})$ in $U_1,$ and combine them together we have
\begin{align*}
L&=U_0+U_1-U_2+U_3\\
&\ge d^2(V^{(l)},V^{(k)})-d(V^{(l)},V^{(k)})K_1-K_2,
\end{align*}
w.p. at least ($1-c/n^2$), where
\begin{align*}
&K_1= C\sqrt{\log n} t \mathcal{V}^{1/2}+Ct^2(p+\log \bar{n})\mathcal{V}Tv^{1/2} ,\\
&K_2= Ct^2 p^2\log n \mathcal{V},
\end{align*}
for some constant $C,c$.\qed

\begin{lem}[\bf  $\tilde{E}_{1,j}^{(k,l)}$ upper bound]
\label{lem:e1_upper_bound}

Suppose $v\in \Gamma_K:= \text{span}\{1_{G_k}:k\in[K]\}^\bot, \|v\|=1.$ Let
\[
E_{1,j}^{(k,l)}=\frac{1}{n_k}\sum_{r\in G_k} d^2(V^{(k)},V_r)+\left[\frac{2}{n_k}\sum_{r\in G_k} d(V^{(k)},V_r) d(V_j,V^{(k)})\right],
\]
and $\tilde{E}_{1,j}^{(k,l)}=E_{1,j}^{(k,l)}/d(V^{(l)},V^{(k)}).$ Then w.p. $\ge 1-n^{-2},$ we have
\[
\sum_{j\in G_l}\tilde{E}_{1,j}^{(k,l)}|v_j|\le C \mathcal{V}^{1/2} p t  \sqrt{n_l} \left(\sum_{j\in G_l} v_j^2 \right)^{1/2},
\]
\end{lem}
\emph{Proof.}
Note $E(\mbox{\rm Tr}(X_r^2))=p^2,E(\sqrt{\mbox{\rm Tr}(X_r^2)})\le \sqrt{E(\mbox{\rm Tr}(X_r^2))}= p$ by Jensen's inequality.  From high-dimension bound for sub-exponential and sub-Gaussian (Hoeffding's inequality and Bernstein’s inequality) we have that w.p. $\ge 1-c/n^2$,
\[
\frac{1}{n_k}\sum_{r\in G_k} d^2(V^{(k)},V_r)=\frac{1}{n_k}\sum_{r\in G_k}\mbox{\rm Tr}(X_r^2 V^{(k)})\le C \mathcal{V} p^2 t^2,
\]
\[
\frac{1}{n_k}\sum_{r\in G_k} d(V^{(k)},V_r)=\frac{1}{n_k}\sum_{r\in G_k}\sqrt{\mbox{\rm Tr}(X_r^2 V^{(k)})}\le C\mathcal{V}^{1/2} p t,
\]
for some constants $C,c$. Suppose that $d(V^{(l)},V^{(k)})\ge C_0t\mathcal{V}^{1/2}\sqrt{\log n}p,$ for some fixed constant $C_0.$  Then we have w.p. $\ge 1-c/n^2$
\[
 d(V_j,V^{(k)})\le d(V_j,V^{(l)})+d(V^{(l)},V^{(k)})\le Ctp\sqrt{\log n}\mathcal{V}^{1/2}+d(V^{(l)},V^{(k)}) \le C d(V^{(l)},V^{(k)}) ,
\]
for some large constant $C$. So we have w.p. $\ge 1-c/n^2$
\[
\sum_{j\in G_l}\tilde{E}_{1,j}^{(k,l)}|v_j|\le C \mathcal{V}^{1/2} p t \sum_{j\in G_l} |v_j|\le C \mathcal{V}^{1/2} p t  \sqrt{n_l} \left(\sum_{j\in G_l} v_j^2 \right)^{1/2},
\]
where  $\mathcal{V}=\max_k \|V^{(k)} \|$, for some large constant $C$.\qed

\begin{lem}[\bf  $\tilde{E}_{2,j}^{(k,l)}$ upper bound]
\label{lem:e2_upper_bound}

Suppose $v\in \Gamma_K:= \text{span}\{1_{G_k}:k\in[K]\}^\bot, \|v\|=1.$ Let
\[
E_{2,j}^{(k,l)}=d^2(V^{(l)},V_j)+2d(V^{(l)},V_j)d(V^{(l)},V^{(k)}).
\]
and $\tilde{E}_{2,j}^{(k,l)}=E_{2,j}^{(k,l)}/d(V^{(l)},V^{(k)}).$ Then w.p. $\ge 1-n^{-1},$ we have
\[
\sum_{j\in G_l}\tilde{E}_{2,j}^{(k,l)}|v_j|\le Ct  \mathcal{V}^{1/2}p(\sqrt{n_l}+\log^2(n))\left(\sum_{j\in G_l} v_j^2 \right)^{1/2},
\]
\end{lem}
\emph{Proof.} First we make the following claim:
\begin{cla}
\label{cla:e2_upper_bound}
Following the above setting, w.p. $\ge 1-cn^{-1},$ we have

\begin{equation}
\label{eqn:e2_upper_bound_1}
\sum_{j\in G_l} d^2(V^{(l)},V_j)|v_j|\le Ct^2\mathcal{V} p^2(\sqrt{n_l}+\log(n)^2)\left(\sum_{j\in G_l} v_j^2 \right)^{1/2}, 
\end{equation}
\begin{equation}
\label{eqn:e2_upper_bound_2}
\sum_{j\in G_l} d(V^{(l)},V_j)|v_j|\le Ct\mathcal{V}^{1/2} p\sqrt{n_l}\left(\sum_{j\in G_l} v_j^2 \right)^{1/2},
\end{equation}
for some large constant $C$.
\end{cla}
If the claim holds, by plugging in the lower bound for $\Delta$ in the assumption, we have
\begin{align*}
\sum_{j\in G_l}\tilde{E}_{2,j}^{(k,l)}|v_j|
&\le \frac{ Ct^2\mathcal{V} p^2(\sqrt{n_l}+\log(n)^2)}{C_0t\mathcal{V}^{1/2}p\sqrt{\log n} }\left(\sum_{j\in G_l} v_j^2 \right)^{1/2}+Ct\mathcal{V}^{1/2} p\sqrt{n_l}\left(\sum_{j\in G_l} v_j^2 \right)^{1/2}\\
&\le Ct  \mathcal{V}^{1/2}p(\sqrt{n_l}+\log^2(n))\left(\sum_{j\in G_l} v_j^2 \right)^{1/2}\\
\end{align*}
\emph{Proof of the claim.} 
First we look at (\ref{eqn:e2_upper_bound_2}):
\[
\sum_{j\in G_l} d(V^{(l)},V_j)|v_j|\le t \mathcal{V}^{1/2} \sum_{j\in G_l}\sqrt{\mbox{\rm Tr}(X_j^2 )}|v_j|.
\]
By Theorem 2.6.3 (General Hoeffding’s inequality) in~\cite{vershynin_2018} we have w.p.$\ge 1-c/n^2,$
\[
\sum_{j\in G_l}\sqrt{\mbox{\rm Tr}(X_j^2 )}|v_j|\le p\sum_{j\in G_l}|v_j|+Cp\sqrt{n_l} \left(\sum_{j\in G_l} v_j^2 \right)^{1/2},
\]
for some constant $C.$ i.e.,  w.p.$\ge 1-c/n^2,$
\[
\sum_{j\in G_l} d(V^{(l)},V_j)|v_j|\le  Ct \mathcal{V}^{1/2}p\sqrt{n_l}\left(\sum_{j\in G_l} v_j^2 \right)^{1/2},
\]
for some constant $C.$ Next we will show (\ref{eqn:e2_upper_bound_1}). First note that $ d^2(V^{(l)},V_j)\le t^2\mathcal{V} \mbox{\rm Tr}(X_j^2 ),$ let
$$G_1(v)=\left|\sum_{j\in G_l}[ \mbox{\rm Tr}(X_j^2 )-\mathbb{E}\mbox{\rm Tr}(X_j^2 )]|v_j|\right|,$$ then
 \[
\sum_{j\in G_l} d^2(V^{(l)},V_j)|v_j|\le t^2 \mathcal{V} G_1(v)+ t^2 \mathcal{V}p^2 \sqrt{n_l}\left(\sum_{j\in G_l} v_j^2 \right)^{1/2}.
\]
W.O.L.G., we may assume $v\in \mathbb{V}:=\{ v\in\Gamma_K: \|v\|=1 \},\ \|G_1 \|_{\mathbb{V}}:=\sup_{v\in \mathbb{V}} | G_1(v)|.$ Then by Theorem 4 in~\cite{10.1214/EJP.v13-521} we know
\[
\mathbb{P}(\|G_1 \|_{\mathbb{V}}\ge 2\mathbb{E} \|G_1 \|_{\mathbb{V}}+s )\le \exp\left( -\frac{s^2}{3\tau_1^2}\right)+3\exp\left( -\frac{s}{3\|M_1\|_{\psi_1}}\right),
\]
where
\[
\tau_1^2=\sup_{v\in \mathbb{V}} \sum_{j\in G_l}  v_j^2 \mathbb{E} [ \mbox{\rm Tr}(X_j^2 )-\mathbb{E} \mbox{\rm Tr}(X_j^2 )]^2\le \mathbb{E}  [\mbox{\rm Tr}(X_j^2 )]^2\le p^4 ,
\]
\[
M_1=\max_{j\in G_l, v\in \mathbb{V}} \left|v_j[ \mbox{\rm Tr}(X_j^2 )-\mathbb{E} \mbox{\rm Tr}(X_j^2 )] \right|\le \max_{j\in G_l} \left|[\mbox{\rm Tr}(X_j^2 )-\mathbb{E} \mbox{\rm Tr}(X_j^2 )] \right|.
\]
By maximal inequality (Lemma 2.2.2 in~\cite{vanderVaart1996}) we have 
\[
\|M_1\|_{\psi_1}\le C \log(n_l)\max_{j\in G_l} \left\|[\mbox{\rm Tr}(X_j^2 )-\mathbb{E} \mbox{\rm Tr}(X_j^2 )] \right\|_{\psi_1}\le C \log(n_l) p^2.
\]
So by choosing $s=C \log^2(n) p^2 $, we have w.p. $\ge 1-c/n,$
\[
G_1(v)\le 2\mathbb{E} \|G_1 \|_{\mathbb{V}}+C \log^2(n) p^2,
\]
for some $C,c.$ On the hand,
\begin{align*}
\mathbb{E} \|G_1 \|_{\mathbb{V}}&= \mathbb{E} \left|\sum_{j\in G_l}[ \mbox{\rm Tr}(X_j^2 )-\mathbb{E}\mbox{\rm Tr}(X_j^2 )]|v_j|\right|\\
&\le \sum_{j\in G_l}\mathbb{E} |\mbox{\rm Tr}(X_j^2 )-\mathbb{E}\mbox{\rm Tr}(X_j^2 )||v_j|\\
&\le 2\mathbb{E} |\mbox{\rm Tr}(X_1^2 )| \sqrt{n_l} \left(\sum_{j\in G_l} v_j^2 \right)^{1/2} \\
&=2p^2 \sqrt{n_l} \left(\sum_{j\in G_l} v_j^2 \right)^{1/2}.
\end{align*}
So w.p. $\ge 1-cn^{-1},$ we have
\[
\sum_{j\in G_l} d^2(V^{(l)},V_j)|v_j|\le Ct^2\mathcal{V} p^2(\sqrt{n_l}+\log(n)^2)\left(\sum_{j\in G_l} v_j^2 \right)^{1/2}.
\]\qed

\end{document}